\documentclass[twoside,11pt]{article}

\usepackage{jmlr2e2}

\usepackage{graphicx}
\usepackage{bbold}
\usepackage{url}

\usepackage{amsmath, amssymb}
\usepackage{algorithmic}
\usepackage{algorithm}

\newcommand{\algo}[5]{
  \begin{algorithm}[top]
    \caption{\tt #1}
    \label{algo:#2}
    \begin{algorithmic}
      \REQUIRE #3 
      \ENSURE #4
      #5
    \end{algorithmic}
  \end{algorithm}
}

\usepackage{graphics}
\newcommand{\FigureEtendue}[3]{
\begin{figure*}[h]
\begin{center}
\resizebox{#3cm}{!}{\includegraphics{./#1.eps}}
\caption{#2}
\label{fig:#1}
\end{center}
\end{figure*}
}

\def\qedhere{}

\newcommand{\transp}[1]{ {#1}^T }
\newcommand{\norm}[2]{\left\|#1\right\|_{#2}}
\DeclareMathOperator*{\spannorm}{span}
\newcommand{\spn}[2]{\spannorm_{#2}\left[#1\right]}

\def\e{\bf{e}}
\def\lp{$L_p$}
\newcommand{\nn}[1]{{\cal N}_1\left( \right)}
\def\1{\mathbb{1}}
\def\lpi{$\lambda$ Policy Iteration}
\def\T{{\cal T}}
\def\O{{\cal O}}
\def\shift{shift}
\def\s{s}
\def\bell{Bellman residual}
\def\b{b}
\def\pbell{Policy Bellman residual}
\def\pb{b'}
\newcommand{\ls}[1]{\limsup_{#1 \rightarrow \infty}~}
\def\td{{\delta}}

\jmlrheading{}{}{}{2011/10}{}{Bruno Scherrer}

\ShortHeadings{Performance Bounds for \lpi}{Bruno Scherrer}
\firstpageno{1}

\begin{document}

\title{Performance Bounds for \lpi\\ and Application to the Game of Tetris}

\author{\name Bruno Scherrer \email bruno.scherrer@inria.fr \\
       \addr Maia Project-Team, INRIA Lorraine\\
       615 rue du Jardin Botanique\\
       54600 Villers-les-Nancy\\      
       FRANCE       
}

\editor{?}

\maketitle

\begin{abstract}
  We consider the discrete-time infinite-horizon optimal control
  problem formalized by Markov Decision Processes
  \citep{puterman,ndp}.  We revisit the work of \citet{ioffe}, that
  introduced \lpi, a family of algorithms parameterized by $\lambda$
  that generalizes the standard algorithms Value Iteration and Policy
  Iteration, and has some deep connections with the Temporal Differences
  algorithm TD($\lambda$) described by \citet{sutton}.  We deepen the
  original theory developped by the authors by providing convergence
  rate bounds which generalize standard bounds for Value Iteration
  described for instance by \cite{puterman}.  Then, the main
  contribution of this paper is to develop the theory of this
  algorithm when it is used in an approximate form and show that this
  is sound. Doing so, we extend and unify the separate analyses
  developped by Munos for Approximate Value Iteration \citep{munosavi}
  and Approximate Policy Iteration \citep{munosapi}.  Eventually, we
  revisit the use of this algorithm in the training of a Tetris
  playing controller as originally done by \citet{ioffe}. We provide
  an original performance bound that can be applied to such an
  undiscounted control problem. Our empirical results are different
  from those of Bertsekas and Ioffe (which were originally qualified
  as ``paradoxical'' and ``intriguing''), and much more conform to
  what one would expect from a learning experiment.  We discuss the
  possible reason for such a difference.
\end{abstract}



\section{Introduction}

We consider the discrete-time infinite-horizon
  optimal control problem formalized by Markov Decision Processes \citep{puterman,ndp}.
We revisit the \lpi{ }algorithm introduced by
\citet{ioffe} (also published in the reference textbook of \cite{ndp}\footnote{The reference \citep{ioffe} being historically anterior to \citep{ndp}, we only refer to the former in the rest of the paper.}), that (as the authors then stated)
\emph{"is primarily
  motivated by the case of large and complex problems where the use of
  approximation is essential''}. It is a family of algorithms
parameterized by $\lambda$ that generalizes the standard Dynamic
Programming algorithms Value Iteration (which corresponds to the case
$\lambda=0$) and Policy Iteration (case $\lambda=1$), and has some
deep connections with the Temporal Differences
  algorithm TD($\lambda$)
that are well known to the Reinforcement Learning community \citep{sutton,ndp}.

In their original paper, \citet{ioffe} show the convergence of \lpi{
}when it is run without error and provide its \emph{asymptotic}
convergence rate. The authors also describe a case study involving an
instance of Approximate \lpi, but neither their paper nor (to the best
of our knowledge) any subsequent work studies the theoretical
soundness of doing so.  In this paper, we extend the theory on this
algorithm in several ways. We derive its \emph{non-asymptotic}
convergence rate when it is run without error.  More importantly, we
develop the theory of \lpi{ }for its main purpose, that is --- recall the
above quote --- when it is run in an approximate form, and prove that
such an approach is sound: we show that the loss of using
the greedy policy with respect to the current value estimate can be
made arbitrarily small by controlling the error made during the
iterations.

The rest of the paper is organized as follows. In Section~\ref{standardalgorithms}, we introduce the framework of Markov
Decision Processes and decribe the two standard
algorithms, Value Iteration and Policy Iteration, along with some of
their state-of-the-art analysis in exact and approximate form.
Section~\ref{lpi} introduces \lpi{ }in an original way that makes its
connection with Value Iteration and Policy Iteration obvious, and
discusses its close connection with Reinforcement Learning. We recall
the main results obtained by \citet{ioffe} (convergence and asymptotic
rate of convergence of the exact algorithm).  At this point of the
paper, we naturally describe how one expects that the properties of
Value Iteration ($\lambda=0$) and Policy Iteration ($\lambda=1$)
described in Section~\ref{standardalgorithms} may translate for
general $\lambda$. The precise statements of our results are the topic
of the next two Sections: Section~\ref{elpi} contains our new results
on Exact \lpi{ }and Section~\ref{alpi} those on Approximate
\lpi\footnote{Section~\ref{alpi} is probably the place where the
  reader familiar with Approximate Dynamic Programming would quickly
  want to jump.}.
Last but not least, Section~\ref{thecasestudy} revisits the empirical part of the work of
\cite{ioffe}, where an approximate version of \lpi{ }is used for
training a Tetris controller.

\paragraph{Notations}

The analysis we describe in this article relies on a few notations, such as several norms and seminorms, that we need to define precisely before we can go further.
Let $X$ be a finite space. Let $u$ denote a real-valued function on $X$, which can be seen as a vector of dimension $|X|$. Let $\e$ denote the vector of which all components are $1$. The vector $\mu$ denotes a distribution on $X$.
We consider the {\bf weighted \lp{ }norm}:
$$
\norm{u}{p,\mu}:=\left(\sum_x \mu(x)|u(x)|^p \right)^{1/p}=(\transp{\mu} |u|^p)^{1/p}
$$
where $|u|^p$ denotes the componentwise absolute value and exponentiation of $u$.
We write $\norm{.}{p}$ the {\bf unweighted \lp{ }norm} (with uniform distribution $\mu$).
The max norm $\norm{.}{\infty}$ is:
$$
\norm{u}{\infty}:=\max_x |u(x)|=\lim_{p \rightarrow \infty}\norm{u}{p}.
$$
We write $\spn{.}{\infty}$ the {\bf span seminorm} (as for instance defined by \citet{puterman}):
$$
\spn{u}{\infty}:=\max_x u(x)-\min_x u(x).
$$
It can be seen that
$$
\spn{u}{\infty}=2 \min_a \norm{u-a \e}{\infty}.
$$
We propose to generalize the span seminorm definition for any $p$ as follows: 
$$
\spn{u}{p,\mu} := 2 \min_a \norm{u-a\e}{p,\mu}
$$
It is clear that it is a seminorm (it is non-negative, it satisfies the triangle inequality and $\spn{au}{*}=|a|\spn{u}{*}$). It is not a norm because it is zero for all constant functions.

The error bounds we derive in this paper are expressed in terms of some span seminorm. The following relations
\begin{equation}
\label{norms}
\left\{
\begin{array}{lclcl}
\spn{u}{p} & \leq & 2\norm{u}{p} & \leq & 2\norm{u}{\infty} \\
\spn{u}{p,\mu} & \leq & 2\norm{u}{p,\mu} & \leq & 2\norm{u}{\infty} \\
\spn{u}{\infty} & \leq & 2\norm{u}{\infty}& &  \\
\end{array}
\right.
\end{equation}
show how to deduce error bounds involving the (more standard) \lp{ }and max norms. Since the span seminorm can be zero for non zero (constant) vectors,  there is no relation that would enable us to derive error bounds in span seminorm from a \lp{ }or a max norm. Bounding an error with the span seminorm is in this sense stronger and this constitutes our motivation for using it.

\section{Framework and Standard Algorithms}

\label{standardalgorithms}

In this section, we begin by providing a short description of the
framework of Markov Decision Processes we consider throughout the
paper. We go on by describing the two main algorithms, Value Iteration and Policy Iteration, for solving the related problem.

\subsection{Markov Decision Processes}

We consider a discrete-time dynamic system whose state transition
depends on a control. We assume that there is a {\bf state space}
$X$ of finite size $N$. When at state $i \in \{1,..,N\}$, the control is chosen from a finite {\bf control space}
$A$. The control $a \in A$ specifies the {\bf transition probability} $p_{ij}(a)$
to the next state $j$. At the $k^\text{th}$ iteration, the system is given a reward $\gamma^k r(i,a,j)$ where $r$ is the instantaneous {\bf reward function},
and $0<\gamma<1$ is a discount factor. The tuple $\langle X, A, p, r, \gamma \rangle$ is called a {\bf Markov Decision Process (MDP)} \citep{puterman,ndp}.

We are interested in stationary deterministic policies, that is
functions $\pi:X \rightarrow A$ which map states into
controls\footnote{Restricting our attention to stationary deterministic policies is not a limitation. Indeed, for the optimality criterion to be defined soon, it can be shown that there exists at least one
stationary deterministic policy which is optimal
\citep{puterman}.}. Writing $i_k$ the state at time $k$, the {\bf value
of policy $\pi$} at state $i$ is defined as the \emph{total expected discounted return}
while following a policy $\pi$ from $i$, that is
\begin{equation}
\label{bellvdef}
v^\pi(i):=\lim_{N\rightarrow \infty} E_\pi \left[ \left. \sum_{k=0}^{N-1}\gamma^k r(i_k,\pi(i_k),i_{k+1}) \right| i_0=i\right]
\end{equation}
where $E_\pi$ denotes the expectation conditional on the fact that the actions are selected with the policy $\pi$ (that is, for all $k$, $i_{k+1}$ is reached from $i_k$ with probability $p_{i_k i_{k+1}}(\pi(i_k))$).
The {\bf optimal value} starting from state $i$ is defined as
$$
v_*(i):= \max_\pi v^\pi(i).
$$

We write $P^\pi$ the $N \times N$ stochastic matrix whose elements are
$p_{ij}(\pi(i))$ and $r^\pi$ the vector whose components are
$\sum_j p_{ij}(\pi(i))r(i,\pi(i),j)$. The value functions $v^\pi$ and $v_*$ can
be seen as vectors on $X$.  It is well known that $v^\pi$ solves the
following Bellman equation:
$$
v^\pi = r^\pi + \gamma P^\pi v^\pi.
$$
The value function $v^\pi$ is a fixed point of the linear
operator $\T^\pi v := r^\pi + \gamma P^\pi v$. As $P^\pi$ is a
stochastic matrix, its eigenvalues cannot be greater than $1$, and consequently
$I-\gamma P^\pi$ is invertible. This implies that 
\begin{equation}
\label{bellvprop}
v^\pi=(I-\gamma P^\pi)^{-1}r^\pi=\sum_{i=0}^\infty (\gamma P^\pi)^i r^\pi.
\end{equation}
It is also well known that $v_*$ satisfies the
following Bellman equation:
$$
v_* = \max_\pi (r^\pi + \gamma P^\pi v_*) = \max_\pi \T^\pi v_*
$$
where the \mbox{max} operator is componentwise.
In other words, $v_*$ is a fixed point of the nonlinear  operator $\T_*
v:=\max_{\pi}\T^\pi v$. For any value vector
$v$, we call a {\bf greedy policy with respect to the value $v$} a
policy $\pi$ that satisfies:
$$
\pi \in \arg\max_{\pi'} \T^{\pi'} v
$$
or equivalently $\T^\pi v = \T_* v$. We write, with some abuse of
notation\footnote{There might be several policies that are greedy with respect to some value $v$.} greedy($v$) any policy that is greedy with respect to
$v$. The notions of optimal value function and greedy policies are
fundamental to optimal control because of the following property: any
policy $\pi_*$ that is greedy with respect to the optimal value is an
{\bf optimal policy} and its value $v^{\pi_*}$ is equal to $v_*$.

The operators $\T^\pi$ and $\T_*$ can be shown to be
$\gamma$-contraction mappings with respect to the max norm. In what
follows we only write what this means for the Bellman operator $\T_*$
but the same holds for $\T^\pi$.  Being a $\gamma$-contraction mapping for the max norm means that for all pairs of vectors $(v,w)$,
$$
\norm{\T_* v - \T_* w}{\infty} \leq \gamma \norm{v-w}{\infty}. 
$$
This ensures that the fixed point $v_*$ of $\cal T$ exists and is unique. Furthermore, for any initial vector $v_0$,
\begin{equation}
\label{limitofcontraction}
\lim_{k \rightarrow \infty} (\T_*)^k v_0 = v_*.
\end{equation}

Given an MDP, standard algorithmic solutions for computing an optimal
value/policy (which dates back to the 1950s, see for instance \citep{puterman} and the
references therein) are Value Iteration and Policy
Iteration. The rest of this section describes both of these algorithms with some of
the relevant properties for the subject of this paper.

\subsection{Value Iteration}

The Value Iteration algorithms for computing the value of a policy
$\pi$ and the value of the optimal policy $\pi_*$ rely on Equation
\ref{limitofcontraction}. Algorithm~\ref{algo:vi} provides a
description of Value Iteration for computing an optimal policy
(replace $\T_*$ by $\T^\pi$ in it and one gets Value Iteration for
computing the value of some policy $\pi$).
\algo{Value Iteration}{vi}
{An MDP, an initial value $v_0$}
{An (approximately) optimal policy}
{
\STATE $k \leftarrow 0$ 
\REPEAT
\STATE 
$v_{k+1} \leftarrow \T_* v_k + \epsilon_{k+1}$ 
~~~~// Update the value
\STATE $k \leftarrow k+1$
\UNTIL{some stopping criterion}
\RETURN{greedy($v_k$)}
} 
In this description, we have introduced a
term $\epsilon_k$ which stands for several possible sources of error
at each iteration: this error might be the computer round off, the
fact that we use an approximate architecture for representing $v$, a
stochastic approximation of $P^{\pi_k}$, etc... or a combination of
these. In what follows, when we talk about the \emph{Exact} version of
an algorithm, this means that $\epsilon_k=0$ for all $k$.

\paragraph{Properties of Exact Value Iteration}

The contraction property induces some
interesting properties for Exact Value Iteration. We have already
mentioned that contraction implies the asymptotic convergence
(Equation~\ref{limitofcontraction}). It can also be inferred that
there is at least a linear rate of convergence: for all reference
iteration $k_0$, and for all $k \geq k_0$,
$$
\norm{v_*-v_k}{\infty} \leq \gamma^{k-k_0}\norm{v_*-v_{k_0}}{\infty}.
$$ 
Even more interestingly, it is possible to derive a performance
bound, that is a bound of the difference between the real value of a
policy produced by the algorithm and the value of the optimal policy
$\pi_*$ \citep{puterman}. Let $\pi_k$ denote the
policy that is greedy with respect to $v_{k-1}$. Then, for all
reference iteration $k_0$, and for all $k \geq k_0$, 
$$
\norm{v_*-v^{\pi_k}}{\infty} \leq \frac{2\gamma^{k-k_0}}{1-\gamma}\norm{\T_* v_{k_0}-v_{k_0}}{\infty}=\frac{2\gamma^{k-k_0}}{1-\gamma}\norm{v_{k_0+1}-v_{k_0}}{\infty}.
$$ 
This fact is of considerable importance computationally since it
provides a stopping criterion: taking $k=k_0+1$, we see that if
$\norm{v_{k_0+1}-v_{k_0}}{\infty} \leq \frac{1-\gamma}{2\gamma}\epsilon$,
then $\norm{v_*-v^{\pi_{k_0+1}}}{\infty} \leq \epsilon$.

It is somewhat less known that the Bellman operators $\T_*$ and $\T^\pi$
are also contraction mapping with respect to the $\mbox{span}_\infty$ seminorm
\citep{puterman}. This means that there exists a variant of the above equation involving the span seminorm instead of the max norm.
For instance, such a fact provides the following stopping
criterion: 
{\begin{proposition}[Stopping Cond. for Exact Value Iteration \citep{puterman}]
\label{stopvi}~\\
If at some iteration $k_0$, the difference between two subsequent iterations satisfies 
$$
\spn{v_{k_0+1}-v_{k_0}}{\infty} \leq \frac{1-\gamma}{\gamma}\epsilon,
$$
then the greedy policy $\pi_{k_0+1}$  with respect to $v_{k_0}$ is $\epsilon$-optimal: 
$\norm{v_*-v^{\pi_{k_0+1}}}{\infty} \leq \epsilon$.
\end{proposition}
}
This latter  stopping criterion is finer since, from the relation between the span seminorm and the norm (Equation~\ref{norms}) it implies the former.

\paragraph{Properties of Approximate Value Iteration (AVI)}

When considering large Markov Decision Processes, one cannot usually
implement an exact version of Value Iteration. In such a case
$\epsilon_k \neq 0$. In general, the algorithm does not converge
anymore but it is possible to study its asymptotic behaviour. The most
well-known result is due to \citet[pp. 332-333]{ndp}: if the approximation erros are
uniformly bounded, the difference between the asymptotic performance
of policies $\pi_{k+1}$ greedy with respect to $v_k$ satisfies
\begin{equation}
\label{ndpavi}
\ls{k}\norm{v_*-v^{\pi_k}}{\infty} \leq \frac{2\gamma}{(1-\gamma)^2}\sup_{k \geq 0}\norm{\epsilon_k}{\infty}.
\end{equation}
\citet{munosapi,munosavi} has recently argued  that, since most supervised
learning algorithms (such as least square regression) that are used in
practice for approximating each iterate of Value Iteration control some \lp{ }norm, it would be more
interesting to have an analogue of the above result where the
approximation error $\epsilon_k$ is expressed in terms of the \lp{
}norm. \citet{munosavi} actually showed how to do this. The
idea is to analyze the \emph{componentwise} asymptotic behaviour of
Approximate Value Iteration, from which it is possible to derive
\lp{ }bounds for any $p$. Write $P_k=P^{\pi_k}$ the stochastic matrix corresponding to the policy $\pi_k$ which is greedy with respect to $v_{k-1}$, $P_*$ the stochastic matrix corresponding to the (unknown) optimal policy $\pi_*$. \citet{munosavi} showed the following lemma:
{\begin{lemma}[Asymptotic Componentwise Performance of AVI \citep{munosavi}]
\label{cavi}~\\
For all  $k>j \ge 0$, the following matrices
\begin{eqnarray*}
Q_{kj} & := & (1-\gamma)(I-\gamma P_k)^{-1} P_k P_{k-1}...P_{j+1}   \\
Q'_{kj} & :=& (1-\gamma)(I-\gamma P_k)^{-1} (P_*)^{k-j} 
\end{eqnarray*}
are stochastic and the asymptotic performance of the policies generated by Approximate Value Iteration satisfies
$$
\ls{k} v_*-v^{\pi_k} \leq \ls{k} \frac{1}{1-\gamma} \sum_{j=0}^{k-1}\gamma^{k-j} \left[ Q_{kj}-Q'_{kj} \right]\epsilon_j.
$$
\end{lemma}
}
From the above componentwise bound, it is possible\footnote{\label{footnote}This result is not explicitely stated by \citet{munosavi}, but using a technique of another of his articles \citep{munosapi}, it can be derived from Lemma~\ref{cavi}. The current paper generalizes this result (in Proposition~\ref{spalpi} page \pageref{spalpi}).} to derive the following \lp{ }bounds.
\begin{proposition}[Asymptotic Performance of AVI (1/2)]
\label{lpavi}~\\
Choose any $p$ and any distribution $\mu$. Consider  the notations of Lemma~\ref{cavi}. Then  
$$
\mu_{kj}:=\frac{1}{2}\transp{\left({Q_{jk}+Q'_{jk}}\right)}\mu
$$
are distributions and 
the asymptotic performance of policies generated by Value Iteration satisfies
$$
\ls{k} \norm{v_*-v^{\pi_k}}{p,\mu}  \leq  \frac{2\gamma}{(1-\gamma)^2} \sup_{k \geq j \geq 0}\norm{\epsilon_k}{p,\mu_{kj}}
$$
\end{proposition}
As the above bounds rely on the partially unknown matrices $Q_{kj}$ and $Q'_{kj}$,
\cite{munosapi,munosavi} introduced some assumption in terms of {\bf concentration coefficient} to remove this dependency. Assume there exists a distribution $\nu$
and a real number $C(\nu)$ such that
\begin{equation}
\label{conccoef}
C(\nu):=\max_{i,j,a} \frac{p_{ij}(a)}{\nu(j)}.
\end{equation}
For instance, if one chooses the uniform law $\nu$, then there always exists such a $C(\nu) \in (1,N)$ where $N$ is the size of the state space (see \citep{munosapi,munosavi} for more discussion on this coefficient). This allows to derive the following performance bounds on the max norm of the loss.
{\begin{proposition}[Asymptotic Performance of AVI (2/2) \citep{munosavi}]
\label{concavi}~\\
Let $C(\nu)$ be the concentration coefficient defined in Equation~\ref{conccoef}. The asymptotic performance of the policies generated by Approximate Value Iteration satisfies
$$
\ls{k}\norm{v_*-v^{\pi_k}}{\infty} \leq \frac{2\gamma \left(C(\nu) \right)^{1/p}}{(1-\gamma)^2}\sup_{k \geq 0}\norm{\epsilon_k}{p,\nu}.
$$
\end{proposition}
}
The main difference between the bounds of Propositions~\ref{lpavi} and~\ref{concavi} and that of Bertsekas and Tsitsiklis (Equation~\ref{ndpavi}) is that the approximation error $\epsilon_k$ is measured by the weighted \lp{ }norm. As $\lim_{p \rightarrow \infty}\norm{.}{p,\mu} \leq \norm{.}{\infty}$ and $\left(C(\nu) \right)^{1/p} \stackrel{p \rightarrow \infty}{\longrightarrow} 1$, Munos's results are strictly finer.

There is generally no guarantee that AVI converges. AVI
converges for specific approximation architectures called \emph{averagers} \citep{gordon} which include state aggregation (see \citep{vanroy} for a very fine approximation bound in this specific case). Also, convergence may
just occur experimentally. Assume that the sequence $(v_k)_{k \ge 0}$ tends to some value $v$. Write
$\pi$ the corresponding greedy policy. Notice this implies that $(\epsilon_k)_{k \ge 0}$
tends to $v-\T_* v$, that is called the {\bf \bell}. The
above bounds can be improved by
a factor $\frac{1}{1-\gamma}$. We know from \citet{baird}
that
$$
\norm{v_*-v^{\pi}}{\infty} \leq \frac{2\gamma}{1-\gamma}\norm{v-\T_* v}{\infty}
$$
and, with the same notations as above,  \cite{munosavi} derived the analogous finer \lp{ }bound:
{\begin{corollary}[Performance of AVI in case of convergence \citep{munosavi}]
\label{avibell}~\\
Let $C(\nu)$ be the concentration coefficient defined in Equation~\ref{conccoef}. Assume that $(v_k)_{k \ge 0}$ tends to some value $v$. Write $\pi$ the corresponding greedy policy. Then
$$
\norm{v_*-v^{\pi}}{\infty} \leq \frac{2\gamma\left(C(\nu) \right)^{1/p} }{1-\gamma}\norm{v-\T_* v}{p,\nu}.
$$
\end{corollary}
}
Eventually, let us mention that \citet{munosavi} and \citet{Farahmand:2010} consider
some \emph{finer} performance bounds (in weighted \lp{ }norm) using
some \emph{finer} concentration coefficients. We won't discuss them
in this paper and we recommend the interested reader to go through
these references for further details.

\subsection{Policy Iteration}

Policy Iteration is an alternative method for computing an optimal policy for an infinite-horizon discounted Markov Decision Process. 
\algo{Policy Iteration}{pi}
{An MDP, an initial policy $\pi_0$}
{An (approximately) optimal policy}
{
\STATE $k \leftarrow 0$
\REPEAT
\STATE 
$v_k \leftarrow 
(I-\gamma P^{\pi_k})^{-1}r^{\pi_k} + \epsilon_k$ ~~~~ // Estimate the value of $\pi_k$
\STATE $\pi_{k+1} \leftarrow \mbox{greedy}(v_k)$ ~~~~~~~~~~~~~~~ // Update the policy
\STATE $k \leftarrow k+1$
\UNTIL{some stopping criterion}
\RETURN{$\pi_k$}
} 
This algorithm is based on the following property: if $\pi$ is some
policy, then any policy $\pi'$ that is greedy with respect to the
value of $\pi$, that is any $\pi'$ satisfying
$\pi'=\mbox{greedy}(v^\pi)$, is better than $\pi$ in the sense that
$v^{\pi'} \geq v^\pi$. Policy Iteration exploits this property in
order to generate a sequence of policies with increasing values. It is
described in Algorithm~\ref{algo:pi}. Note that we use the analytical
form of the value of a policy given by Equation~\ref{bellvprop}. Also,
as for Value Iteration, our description includes a potential error
$\epsilon_k$ term each time the value of a policy is estimated.

\paragraph{Properties of Exact Policy Iteration}

When the state space and the control spaces are finite, 
Exact Policy Iteration converges to an optimal policy $\pi_*$ in
a finite number of iterations \citep{puterman,ndp}. In infinite state spaces, if the function $v \mapsto
P^{\mbox{greedy}(v)}$ is Lipschitz, then it can be shown that
Policy Iteration has a quadratic convergence rate \citep{puterman}.
However, to our knowledge, and contrary to Value Iteration,
finite-time stopping conditions such as that of Proposition
\ref{stopvi} are not widely known for Policy Iteration, though they
appear implicitely in some recent works on Approximate Policy Iteration
\citep{Antos:2007a,Antos:2008,Farahmand:2010,Lazaric:2010}.

\paragraph{Properties of Approximate Policy Iteration (API)}

For problems of interest, one usually uses Policy Iteration in an
approximate form, that is with $\epsilon_k \neq 0$. Results similar to
those presented for Approximate Value Iteration exist for
Approximate Policy Iteration. As soon as there is some error
$\epsilon_k \neq 0$, the algorithm does not necessarily converge
anymore but there is an analogue of Equation~\ref{ndpavi} which is
also due to \citet[Prop 6.2 p. 276]{ndp}: if the approximation errors are uniformly
bounded, then
the difference between the asymptotic performance of policies
$\pi_{k+1}$ greedy with respect to $v_k$ and the optimal policy is
\begin{equation}
\label{ndpapi}
\ls{k}\norm{v_*-v^{\pi_k}}{\infty} \leq \frac{2\gamma}{(1-\gamma)^2}\sup_{k \ge 0}\norm{\epsilon_k}{\infty}.
\end{equation}
As for Value Iteration, Munos has extended this result so that one can
get bounds involving the \lp{ }norm. He also showed how to relate the
performance analysis to the Bellman residual $v_k-\T^{\pi_k}v_k$ that
says how much $v_k$ approximates the real value of the policy $\pi_k$;
this is for instance interesting when the evaluation step of Approximate Policy
Iteration involves the minimization of the norm of this Bellman residual (see \citep{munosapi}). It is
important to note that this Bellman residual is different from the one
we introduced in the previous section (we then considered $v_k-\T_* v_k=
v_k-\T^{\pi_{k+1}} v_k$ where $\pi_{k+1}$ is greedy with respect to
$v_k$). To avoid confusion, and because it is related to some
specific policy, we call $v_k-\T^{\pi_k}v_k$ the {\bf \pbell}.
Munos started by deriving a componentwise analysis. Write $P_k=P^{\pi_k}$ the stochastic matrix corresponding to the policy $\pi_k$ which is greedy with respect to $v_{k-1}$, $P_*$ the stochastic matrix corresponding to the (unknown) optimal policy $\pi_*$.
{\begin{lemma}[Asymptotic Componentwise Performance of API \citep{munosapi}]
\label{capi}~\\
The following matrices
\begin{eqnarray*}
R_k & := & (1-\gamma)^2(I-\gamma P_*)^{-1} P_{k+1}(I-\gamma P_{k+1})^{-1} \\
R'_k & := & (1-\gamma)^2(I-\gamma P_*)^{-1} \left[P_*+\gamma P_{k+1}(I-\gamma P_{k+1})^{-1}P_k \right]\\
R''_k & := & (1-\gamma)^2(I-\gamma P_*)^{-1} P_*(I-\gamma P_k)
\end{eqnarray*}
are stochastic and the asymptotic performance of the policies generated by Approximate Policy Iteration satisfies
\begin{eqnarray*}
\ls{k} v_*-v^{\pi_k} & \leq & \frac{2\gamma}{(1-\gamma)^2} \ls{k} \left[R_k-R'_k \right]\epsilon_k \\
\ls{k} v_*-v^{\pi_k} & \leq & \frac{2\gamma}{(1-\gamma)^2} \ls{k} \left[R_k-R''_k \right](v_k-\T^{\pi_k} v_k).
\end{eqnarray*}
\end{lemma}
}
As for Value Iteration, the above componentwise bound leads to the following \lp{ }bounds.
\begin{proposition}[Asymptotic Performance of API (1/2) \citep{munosapi}]
\label{lpapi}~\\
Choose any $p$ and any distribution $\mu$. Consider the notations of Lemma~\ref{capi}. For all $k \ge 0$, 
\begin{align}
\mu_k := \frac{1}{2}\transp{\left({R_{k}+R'_{k}}\right)}\mu
\mbox{~~and~~} \mu'_k:= \frac{1}{2}\transp{\left({R_{k}+R''_{k}}\right)}\mu \nonumber
\end{align}
are distributions and the asymptotic performance of the policies generated by Approximate Policy Iteration satisfies
\begin{eqnarray*}
\ls{k} \norm{v_*-v^{\pi_k}}{p,\mu} &  \leq & \frac{2\gamma}{(1-\gamma)^2} \ls{k} \norm{\epsilon_k}{p,\mu_k} \\
\mbox{and~~}\ls{k} \norm{v_*-v^{\pi_k}}{p,\mu} &  \leq & \frac{2\gamma}{(1-\gamma)^2} \ls{k} \norm{v_k-\T^{\pi_k} v_k}{p,\mu'_k}.
\end{eqnarray*}
\end{proposition}
Using the concentration coefficient $C(\nu)$ introduced in the
previous section (Equation~\ref{conccoef}), it is also possible to
show\footnote{Similarly to footnote~\ref{footnote}, this result is not explicitely stated by \cite{munosapi}
but using techniques of another of his articles \citep{munosavi}, it can be derived from Lemma~\ref{capi}. The current paper
anyway generalizes this result (in Proposition~\ref{calpi} page
\pageref{calpi}).} the following $L_\infty$/\lp{ }bounds:
{\begin{proposition}[Asymptotic Performance of API (2/2)]
\label{concapi}~\\
Let $C(\nu)$ be the concentration coefficient defined in Equation~\ref{conccoef}. The asymptotic performance of the policies generated by Approximate Policy Iteration  satisfies
\begin{eqnarray*}
\ls{k}\norm{v_*-v^{\pi_k}}{\infty} & \leq & \frac{2\gamma \left(C(\nu) \right)^{1/p}}{(1-\gamma)^2}\ls{k}\norm{\epsilon_k}{p,\nu} \\
\ls{k}\norm{v_*-v^{\pi_k}}{\infty} & \leq & \frac{2\gamma \left(C(\nu) \right)^{1/p}}{(1-\gamma)^2}\ls{k}\norm{v_k-\T^{\pi_k} v_k}{p,\nu}.
\end{eqnarray*}
\end{proposition}
} 
Again, the bounds of Propositions~\ref{lpapi} and~\ref{concapi} with respect to the approximation error $\epsilon_k$
are finer than that of Bertsekas and Tsitsiklis (Equation~\ref{ndpapi}).
Compared to the similar result for Approximate Value Iteration (Propositions
\ref{lpavi} and~\ref{concavi}) where the bound depends on a \emph{uniform} error bound ($\forall k, \norm{\epsilon_k}{p,\nu} \leq \epsilon$), the above
bounds have the nice property that they only depend on
\emph{asymptotic} errors/residuals.

Finally, as for Approximate Value Iteration, a better bound (by a factor $\frac{1}{1-\gamma}$) might be obtained
if the sequence of policies happens to converge. It can be shown \citep[Remark 4 page 7]{munosapi} that:
{\begin{corollary}[Performance of API in case of convergence]
\label{apibell}~\\
Let $C(\nu)$ be the concentration coefficient defined in Equation~\ref{conccoef}. If the sequence of policies $(\pi_k)$ converges to some $\pi$, then
\begin{eqnarray*}
v_*-v^{\pi} & \leq & \frac{2\gamma \left(C(\nu) \right)^{1/p}}{1-\gamma}\ls{k}\norm{\epsilon_k}{p,\nu} \\
v_*-v^{\pi} & \leq & \frac{2\gamma \left(C(\nu) \right)^{1/p}}{1-\gamma}\ls{k}\norm{v_k-\T^{\pi_k} v_k}{p,\nu}.
\end{eqnarray*}
\end{corollary}
} 

After this tour of results for Value and Policy Iteration, we now
introduce the algorithm studied in this paper.

\section{\lpi}

\label{lpi}

Though all the results we have emphasized so far are
strongly related (and even sometimes identical, compare Equations
\ref{ndpavi} and~\ref{ndpapi}), they were surprisingly proved
independently. In this section, we describe the family of algorithms
``\lpi''\footnote{\label{tdbpi}It was also called ``Temporal Difference-Based Policy Iteration'' in the original paper, but we take the name \lpi, as it was the name picked by most subsequent works.} 
introduced by \cite{ioffe} parameterized by a coefficient $\lambda \in
(0,1)$, that generalizes them both. When $\lambda=0$, \lpi{ }
reduces to Value Iteration while it reduces to Policy Iteration
when $\lambda=1$. We also recall the fact discussed by
\citet{ioffe} that \lpi{ }draws some connections with Temporal
Difference algorithms  \citep{sutton}.  

\subsection{The Algorithm}

We begin by giving some intuition about how one can make a connection
between Value Iteration and Policy Iteration. For simplicity, let us temporarily
forget about the error term~$\epsilon_k$. At first sight, Value Iteration builds a sequence of value functions and Policy Iteration a
sequence of policies. In fact, both algorithms can  be seen as updating
a sequence of value-policy pairs.  With some little rewriting --- by
decomposing the (nonlinear) Bellman operator $\T_*$ into \emph{(i)} the
maximization step and \emph{(ii)} the application of the (linear) Bellman
operator --- it can be seen that each iterate of Value Iteration is equivalent to
the two following updates: $$
\left\{
\begin{array}{rcl}
\pi_{k+1} & \leftarrow & \mbox{greedy}(v_{k})\\
v_{k+1} &\leftarrow& \T^{\pi_{k+1}} v_{k}
\end{array}
\right.
\Leftrightarrow
\left\{
\begin{array}{rcl}
\pi_{k+1} & \leftarrow & \mbox{greedy}(v_{k})\\
v_{k+1} &\leftarrow&  r^{\pi_{k+1}}+\gamma P^{\pi_{k+1}} v_{k}.
\end{array}
\right.
$$
The left hande side of the above equation uses the operator $\T^{\pi_{k+1}}$ while the right hande side uses its definition.
Similarly --- by inverting in Algorithm~\ref{algo:pi} the order of \emph{(i)} the estimation of the value of the current policy and \emph{(ii)} the update of the policy, and by using the fact that the value of the policy $\pi_{k+1}$ is the fixed point of $\T^{\pi_{k+1}}$ (Equation~\ref{limitofcontraction}) --- it can be argued that every iteration of Policy Iteration does the following: 
$$
\left\{
\begin{array}{rcl}
\pi_{k+1} & \leftarrow & \mbox{greedy}(v_{k})\\
v_{k+1} & \leftarrow & (\T^{\pi_{k+1}})^\infty v_{k}
\end{array}
\right.
\Leftrightarrow
\left\{
\begin{array}{rcl}
\pi_{k+1} & \leftarrow & \mbox{greedy}(v_{k})\\
v_{k+1} & \leftarrow & (I-\gamma P^{\pi_{k+1}})^{-1}r^{\pi_{k+1}}.
\end{array}
\right.
$$
This rewriting makes both algorithms look close to
each other. Both can be seen as having an estimate $v_k$ of the value
of policy $\pi_k$, from which they deduce a potentially better policy
$\pi_{k+1}$. The corresponding value $v^{\pi_{k+1}}$ of this better
policy may be regarded as a target which is  tracked by the
next estimate $v_{k+1}$. The difference is in the update that enables
to go from $v_k$ to $v_{k+1}$: while Policy Iteration directly
\emph{jumps to} the value of $\pi_{k+1}$ (by applying the Bellman
operator $\T^{\pi_{k+1}}$ an infinite number of times), Value
Iteration only \emph{makes one step} towards it (by applying
$\T^{\pi_{k+1}}$ only once). From this common view of Value Iteration,
it is natural to introduce the well-known Modified Policy Iteration algorithm \citep{Puterman:1978} which \emph{makes $n$ steps} at each update:
$$
\left\{
\begin{array}{rcl}
\pi_{k+1} & \leftarrow & \mbox{greedy}(v_{k})\\
v_{k+1} & \leftarrow & (\T^{\pi_{k+1}})^n v_{k}
\end{array}
\right.
\Leftrightarrow
\left\{
\begin{array}{rcl}
\pi_{k+1} & \leftarrow & \mbox{greedy}(v_{k})\\
v_{k+1} & \leftarrow & \left[I+...+(\gamma P^{\pi_{k+1}})^{n-1}\right] r^{\pi_{k+1}} + (\gamma P^{\pi_{k+1}})^n v_k.
\end{array}
\right.
$$
\FigureEtendue{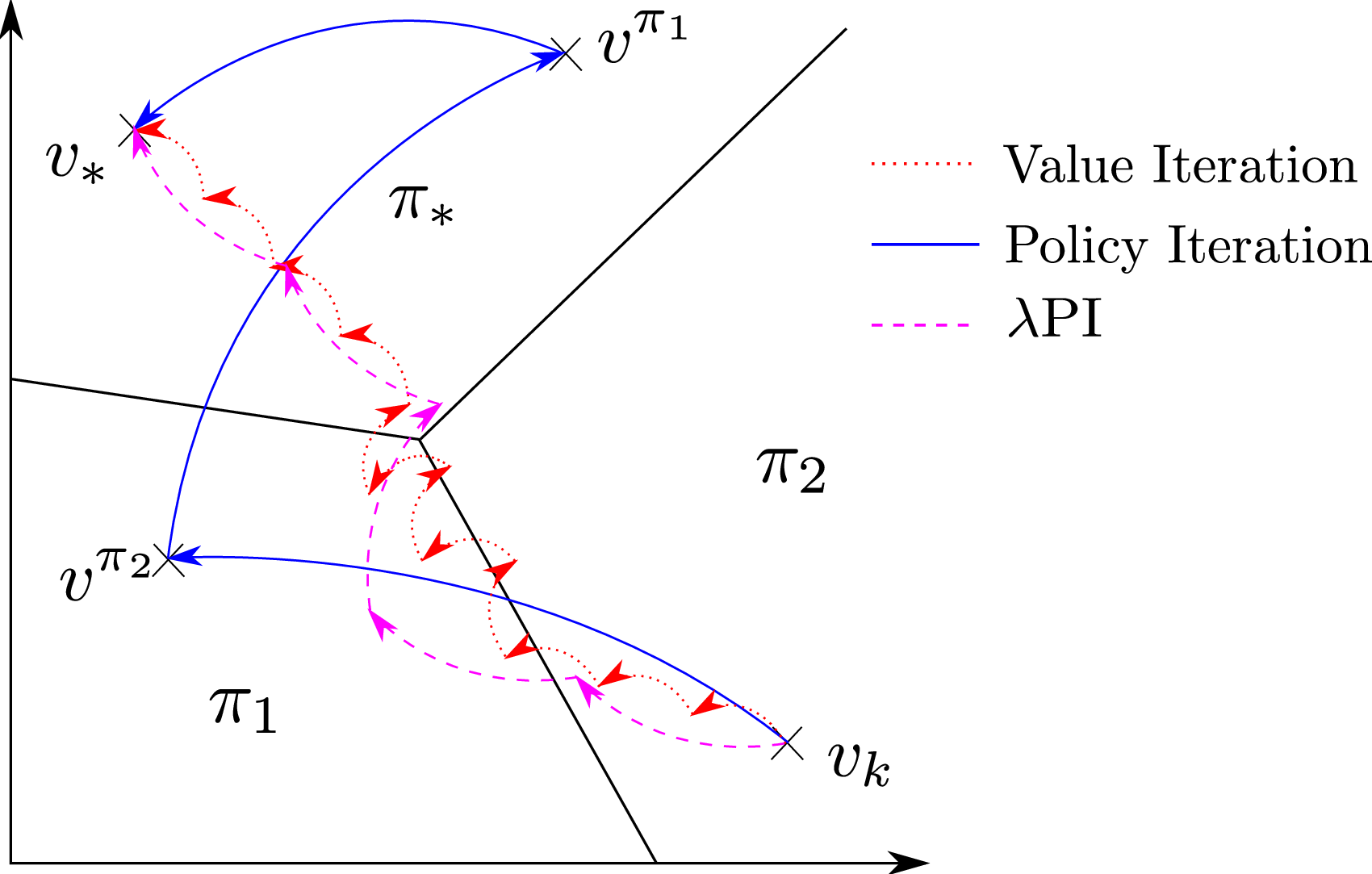}{{\bf Visualizing \lpi{ }on the greedy
    partition sketch:} Following \citep[page 226]{ndp}, one can
  decompose the value space as a collection of polyhedra, such that
  each polyhedron corresponds to a region where one policy is greedy.
  This is called the \emph{greedy partition}. In the above example, there are only 3
  policies, $\pi_1$, $\pi_2$ and $\pi_*$. $v_k$ is the initial value.
  greedy$(v_k)=\pi_2$, greedy$(v^{\pi_2})=\pi_1$, and
  greedy$(v^{\pi_1})=\pi_*$. Therefore (1-)Policy Iteration generates
  the sequence $((\pi_2,v^{\pi_2}), (\pi_1,v^{\pi_1}),
  (\pi_*,v^{\pi_*}))$. Value teration (or $0$ Policy Iteration) starts
  by slowly updating $v_k$ towards $v^{\pi_2}$ until it crosses the
  boundary $\pi_1/\pi_2$, after which it tracks alternatively
  $v^{\pi_1}$ and $v^{\pi_2}$, until it reaches the $\pi_*$ part. In
  other words, Value Iteration makes small steps. \lpi{ }is
  intermediate between the two: it makes steps of which the steps length is
  related to $\lambda$. On the above sketch, \lpi{ }gets to the greedy partition of the optimal policy $\pi_*$ in 3 steps, as did Policy Iteration.} {10}
The above common view is actually here interesting because
it also leads to a natural introduction of \lpi. \lpi{ }is
doing a \emph{$\lambda$-adjustable step} towards the value of
$\pi_{k+1}$:
$$
\hspace{-7.5cm}\left\{
\begin{array}{rcl}
\pi_{k+1} & \leftarrow & \mbox{greedy}(v_{k})\\
v_{k+1} & \leftarrow & (1-{\lambda})\sum_{j=0}^\infty {\lambda}^j (\T^{\pi_{k+1}})^{j+1} v_{k} 
\end{array}
\right. 
$$
$$
\hspace{5cm}
\Leftrightarrow  \left\{
\begin{array}{rcl}
\pi_{k+1} & \leftarrow & \mbox{greedy}(v_{k})\\
v_{k+1} & \leftarrow & (I-{\lambda}\gamma P^{\pi_{k+1}})^{-1}(r^{\pi_{k+1}}+(1-{\lambda})\gamma P^{\pi_{k+1}}v_{k}) 
\end{array}
\right.
$$
\begin{remark}
The equivalence between the left and the right representation of \lpi{ }needs here to be proved. For all $k \geq 0$ and all function $v$, \citet{ioffe} introduce the following operator\footnote{The equivalence between Equations~\ref{mdef1} and~\ref{mdef2} follows trivially from the definition of $\T^{\pi_{k+1}}$.}
\begin{eqnarray}
M_k {\bf v} & := & (1-\lambda)\T^{\pi_{k+1}}v_k + \lambda \T^{\pi_{k+1}}{\bf v}\label{mdef1} \\
& = &r^{\pi_{k+1}}+(1-\lambda)\gamma P^{\pi_{k+1}}v_k + \lambda \gamma P^{\pi_{k+1}}{\bf v} \label{mdef2}
\end{eqnarray}
and prove that
\begin{itemize}
\item $M_k$ is a contraction mapping of modulus $\lambda\gamma$ for the max norm ;
\item The next iterate $v_{k+1}$ of \lpi{ }is the (unique) fixed point of $M_k$.
\end{itemize}
The left representation of \lpi{ }is obtained by ``unrolling'' Equation~\ref{mdef1} an infinite number of times, while the right one is obtained by using Equation~\ref{mdef2} and solving the linear system $v_{k+1}=M_k v_{k+1}$.
\end{remark}

Informally, the parameter $\lambda$ (or $n$ in the case of Modified Policy Iteration) can be seen as adjusting the size of the step for tracking the target $v^{\pi_{k+1}}$ (see Figure~\ref{fig:greedyexact}): the bigger the value, the longer the step.  
Formally, \lpi{ }(consider the above left hand side) consists in doing a geometric average of parameter $\lambda$ of the different numbers of applications of the Bellman operator $(\T^{\pi_{k+1}})^{j}$ to $v_{k}$. The right hand side is here interesting because it clearly shows that \lpi{ }generalizes Value Iteration (when $\lambda=0$) and Policy Iteration (when $\lambda=1$).
The operator $M_k$ gives some insight on how one may concretely implement
one iteration of \lpi: it can for instance be done through a Value
Iteration-like algorithm which applies $M_k$ iteratively. Also, the
fact that its contraction factor $\lambda \gamma$ can be tuned is of particular
importance because finding the corresponding fixed point can be much
easier than that of $\T^{\pi_{k+1}}$, which is only $\gamma$-contracting.

In order to describe the \lpi{ }algorithm, it is useful to introduce an operator that corresponds to computing the fixed point of $M_k$.
For any value $v$ and any policy $\pi$, define:
\begin{eqnarray}
\T^{\pi}_\lambda v & := & v + (I-\lambda\gamma P^{\pi})^{-1}(\T^{\pi}v - v) \label{form1} \\
& = & (I-\lambda\gamma P^{\pi})^{-1}(r^{\pi}+ (1-\lambda)\gamma P^{\pi} v ) \label{form3}  \\
& = & (I-\lambda\gamma P^{\pi})^{-1}(\lambda r^{\pi}+ (1-\lambda)\T^{\pi} v ). \label{form4}
\end{eqnarray}
Equation~\ref{form3} indeed amounts to solve Equation~\ref{mdef2} defining $M_k$. The other two formulations are equivalent up to some little linear algebra manipulations. 

\lpi{ }is formally described in Algorithm~\ref{algo:lpi}.
\algo{\lpi}{lpi}
{An MDP, $\lambda \in (0,1)$, an initial value $v_0$}
{An (approximately) optimal policy}
{
\STATE $k \leftarrow 0$
\REPEAT
\STATE $\pi_{k+1} \leftarrow \mbox{greedy}(v_k)$ ~~~~~~~~ // Update the policy
\STATE $v_{k+1} \leftarrow \T^{\pi_{k+1}}_\lambda v_k + \epsilon_{k+1}$ ~~ // Update the estimate of the value of policy $\pi_{k+1}$
\STATE $k \leftarrow k+1$
\UNTIL{some convergence criterion}
\RETURN{greedy($v_k$)}
} 
Once again, our description includes a potential error term each
time the value is updated.  Even with this error term, it is
straightforward to see that the algorithm reduces to Value Iteration
(Algorithm~\ref{algo:vi}) when $\lambda=0$ and to Policy Iteration\footnote{Policy Iteration
starts with an initial policy while \lpi{ }starts with some initial
value. To be precise, $1$ Policy Iteration starting with $v_0$ is
equivalent to Policy Iteration starting with the greedy policy with
respect to $v_0$.}
(Algorithm~\ref{algo:pi}) when $\lambda=1$.

\paragraph{Relation with Reinforcement Learning}

\begin{figure}[t]
\begin{minipage}[l]{.49\linewidth}
\begin{center}
  \includegraphics[width=7cm]{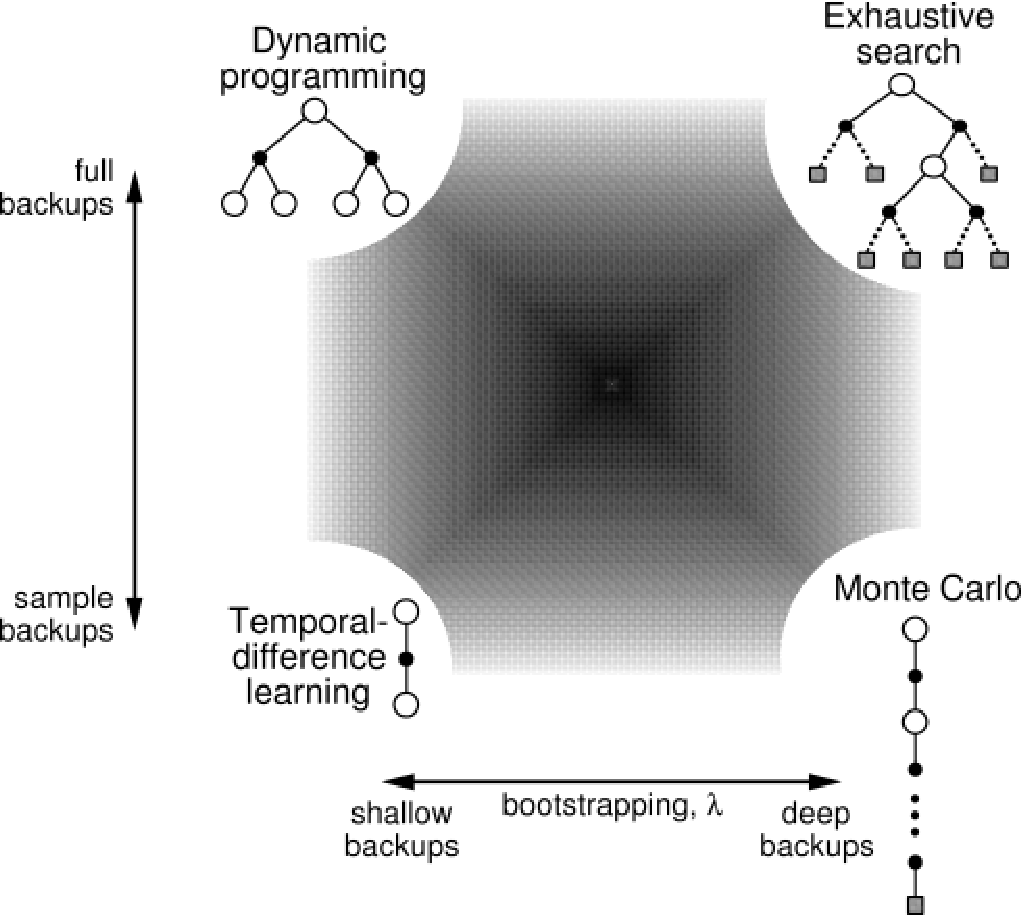}
\end{center}
\end{minipage}	
\begin{minipage}[r]{.49\linewidth}
\begin{center}
  \includegraphics[width=6cm]{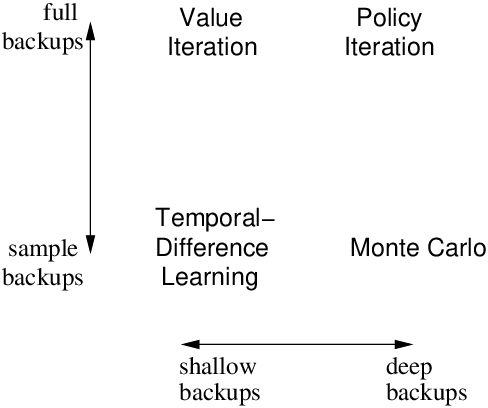}
\end{center}
\end{minipage}
\caption{\label{fig:lpi}{\bf \lpi, a fundamental algorithm for
    Reinforcement Learning:} The left drawing, taken from chapter 10.1
  the book of \citet{sutton}, represents \emph{``two of the most
    important dimensions''} of Reinforcement Learning methods. The
  vertical axis corresponds to whether one does full backup (exact
  computation of the expectations) or stochastic approximation
  (estimation through samples). The horizontal axis corresponds to the
  depth of the backups, and is (among other things) controlled by the
  parameter $\lambda$. On the right, is an original picture of \lpi,
  along the same dimensions. It is interesting to notice that
  \citet{sutton} comment their drawing as follows: \emph{``At three of
    the four corners of the space are the three primary methods for
    estimating values: DP, TD, and Monte Carlo''}. They do not
  recognize the fourth corner as one of the Reinforcement Learning
  \emph{primary methods}. The natural representation of \lpi{
  }actually suggests a modification of the sketch which is
  particularly meaningful: Policy Iteration, which consists in
  computing the value of the current policy, is the \emph{deepest
    backup method}, and can be considered as the batch version of
  Monte Carlo.}
\end{figure}
The definition of the operator $\T_\lambda^\pi$ given by Equation
\ref{form3} is the form we have used for the introduction of \lpi{ }as
an intermediate algorithm between Value Iteration and Policy Iteration. The
equivalent form given by Equation~\ref{form1} can be used to make a
connection with the TD($\lambda$) algorithms\footnote{TD stands for
Temporal Difference. As we have mentionned in Footnote~\ref{tdbpi}, \lpi{ }was originally also called ``Temporal Difference Based
Policy Iteration'' and the presentation of \citet{ioffe} starts from the
formulation of Equation~\ref{form1} (which is close to TD($\lambda$)), and afterwards makes the connection with
Value Iteration and Policy Iteration.} \citep{sutton}. Indeed, through Equation~\ref{form1}, the evaluation phase of \lpi{ }can be seen as an incremental additive procedure:
$$
v_{k+1} \leftarrow v_k + \Delta_k
$$
where 
$$
\Delta_k:=(I-\lambda\gamma P^{\pi_{k+1}})^{-1}(\T^{\pi_{k+1}} v_k-v_k )$$
is zero if and only if the value $v_k$ is equal to the optimal value $v_*$. 
It can be shown (see
\cite{ioffe} for a proof or simply look at the equivalence between
Equations~\ref{bellvdef} and~\ref{bellvprop} for an intuition) that the vector $\Delta_k$
has components given by:
\begin{equation}
\label{tempdif}
\Delta_k(i)=\lim_{N \rightarrow \infty}E_{\pi_{k+1}} \left[ \left. \sum_{j=0}^{N-1}(\lambda\gamma)^j \td_k(i_j,i_{j+1}) \right| i_0=i\right] 
\end{equation}
with
$$
\td_k(i,j)  :=  r(i,\pi_{k+1}(i),j)+\gamma v(j)-v(i) 
$$
being the temporal difference associated to transition $i \rightarrow
j$, as defined by \citet{sutton}.  When one uses a stochastic
approximation of \lpi, that is when the expectation $E_{\pi_{t+1}}$ is
approximated by sampling, \lpi{ }reduces to the algorithm
TD($\lambda$) which is described in chapter 7 of \cite{sutton}.
In particular, when $\lambda=1$, the terms in the above sum collapse and become the exact discounted return:
\begin{align}
 \sum_{j=0}^{N-1}\gamma^j \td_k(i_j,i_{j+1}) &=  \sum_{j=0}^{N-1}\gamma^j \left[ r(i_j,\pi_{k+1}(i_j),i_{j+1})+\gamma v(i_{j+1})-v(i_j) \right] \nonumber \\
&= \sum_{j=0}^{N-1}\gamma^j  r(i_j,\pi_{k+1}(i_j),i_{j+1}) + \gamma^N v(i_N) \nonumber
\end{align}
and the stochastic approximation matches the Monte-Carlo method.
Also, \citet{ioffe} show that Approximate TD($\lambda$) with a linear
feature architecture, as described in chapter 8.2 of \cite{sutton},
corresponds to a natural Approximate version of \lpi{ }where the value
is updated by least square fitting using a gradient-type iteration
after each sample. Last but not least, the reader might notice that
the ``unified view'' of Reinforcement Learning algorithms which is
depicted in chapter 10.1 of \cite{sutton}, and which is reproduced in
Figure~\ref{fig:lpi}, is in fact a picture of \lpi.

\subsection{Analysis of Exact \lpi}

To our knowledge, little has been done concerning the analysis of \lpi: the only results available concern the Exact case (when
$\epsilon_k=0$).  Define the following factor
\begin{equation}
\label{betadef}
\beta=\frac{(1-\lambda)\gamma}{1-\lambda\gamma}.
\end{equation}
We have $0 \leq \beta \leq \gamma < 1$. If $\lambda=0$ (Value
Iteration) then $\beta=\gamma$, and if $\lambda=1$ (Policy Iteration)
then $\beta=0$. 
In the original article introducing \lpi, \citet{ioffe} show the convergence and provide an asymptotic rate of convergence:
{\begin{proposition}[Convergence of Exact $\lambda$PI \citep{ioffe}]
\label{lpiorig}~\\
If the discount factor $\gamma<1$, then $v_k$ converges to $v_*$. Furthermore, after some index $k_*$, the rate of convergence is linear in $\beta$ as defined in Equation~\ref{betadef}, that is
$$
\forall k \geq k_*,~~\|v_{k+1}-v_*\| \leq \beta \|v_k-v_*\|.
$$
\end{proposition}
}
By making $\lambda$ close to 1, $\beta$ can be arbitrarily close to
$0$ so the above rate of convergence might look overly impressive. This needs
to be put into perspective: the index $k_*$ is the index after which
the policy $\pi_k$ does not change anymore (and is equal to the
optimal policy $\pi_*$). As we said when we introduced the
algorithm, $\lambda$ controls the speed at which one wants $v_k$ to ``track the target'' $v^{\pi_{k+1}}$; when $\lambda=1$, this is done in one step (and if $\pi_{k+1}=\pi_*$ then $v_{k+1}=v_*$). 

\section{Overview of our Results and Main Proof Ideas}

\label{resultsandproofidea}

Now that we have described the algorithms and some of their known
properties, motivating the remaining of this paper is
straightforward. \lpi{ }is conceptually nice  since it
generalizes the two most well-known algorithms for solving discounted
infinite-horizon Markov Decision Processes. The natural question that
arises is whether one can generalizes the results we have described
so far to \lpi{ }(uniformly for all $\lambda$). The answer is yes:
\begin{itemize}
\item we shall derive a componentwise analysis of Exact and Approximate \lpi;
\item we shall characterize the \emph{non-asymptotic} convergence rate of Exact \lpi{ }  (Proposition~\ref{lpiorig} only showed the \emph{asymptotic} linear convergence) and shall generalize the stopping criterion described for Value Iteration (Proposition~\ref{stopvi});
\item we shall give bounds of the asymptotic error of Approximate \lpi{ }with respect to the \emph{asymptotic} approximation error, \bell, and \pbell, that generalize Lemmas~\ref{cavi} and~\ref{capi}, and Propositions~\ref{lpavi},~\ref{concavi},~\ref{lpapi} and~\ref{concapi}; our analysis actually implies that doing Approximate \lpi{ }is sound (when the approximation error tends to 0, the algorithm recovers the optimal solution)
\item we shall provide specific (better) bounds for the case when the value or the policy converges, which generalizes Corollaries~\ref{avibell} and~\ref{apibell};
\item interestingly, we shall provide all our results using the span seminorms we have introduced at the beginning of the paper, and using the relations between this span semi-norms and the standard \lp{ }norms (Equation~\ref{norms}), it can be seen that our results are in this respect slightly stronger than all the previously described results.
\end{itemize}
Conceptually, we provide a unified vision (unified proofs, unified
results) for all the mentionned algorithms.

\subsection{On the Need for a New Proof Technique}

In the literature, lines of analysis are different for Value Iteration and Policy
Iteration. Analyses of Value Iteration are based on the fact that it  computes the fixed point of the Bellman operator
which is a $\gamma$-contraction mapping in max norm (see for instance \citep{ndp}). Unfortunately, it can be shown that the operator by
which Policy Iteration updates the value from one iteration to the
next is in general not a contraction in max norm.  In fact, this observation can be drawn for \lpi{
}as soon as it does not reduce to Value Iteration:
\begin{proposition}{
If $\lambda>0$, there exists no norm for which the operator by
which \lpi{ }updates the value from one iteration to the next is a
contraction.
}
\end{proposition}
\begin{proof}
To see this, consider the deterministic MDP (shown in Figure~\ref{simplemdp})
with two states $\{1,2\}$ and two actions $\{change,stay\}$: $r_1=0$,
$r_2=1$,
$P_{change}(s_2|s_1)=P_{change}(s_1|s_2)=P_{stay}(s_1|s_1)=P_{stay}(s_2|s_2)=1$.
\begin{figure}[t]
    \begin{center}
      \includegraphics[width=3cm]{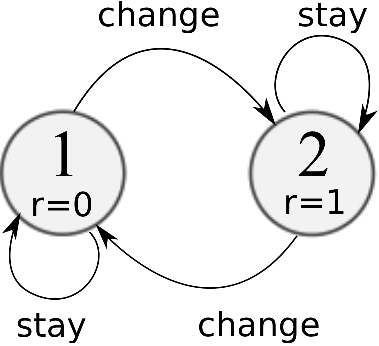} 
    \end{center}
  \caption{\label{simplemdp} This simple deterministic MDP is used to show that \lpi{ }cannot be analysed in terms of contraction (see text for details).}  
\end{figure}
Consider the following two value functions $v=(\epsilon,0)$ and
$v'=(0,\epsilon)$ with $\epsilon>0$. Their corresponding greedy
policies are $\pi=(stay,change)$ and $\pi'=(change,stay)$. Then, we
can compute the next iterates of $v$ and $v'$ (using Equation~\ref{form3}):
\begin{eqnarray*}
r^\pi+(1-\lambda\gamma)P^{\pi}v & = & \begin{pmatrix}(1-\lambda)\gamma\epsilon \\1+(1-\lambda)\gamma\epsilon\end{pmatrix},\\
\T^\pi_\lambda v & = & \begin{pmatrix} \frac{(1-\lambda)\gamma\epsilon}{1-\lambda\gamma} \\1+\frac{(1-\lambda)\gamma\epsilon}{1-\lambda\gamma} \end{pmatrix}, \\
r^{\pi'}+(1-\lambda\gamma)P^{\pi'}v' & = & \begin{pmatrix} (1-\lambda)\gamma\epsilon \\1+(1-\lambda)\gamma\epsilon \end{pmatrix}, \\
\mbox{and~~}\T^{\pi'}_\lambda v' & = &  \begin{pmatrix} \frac{1+(1-\lambda)\gamma\epsilon}{1-\lambda\gamma}-1 \\ \frac{1+(1-\lambda)\gamma\epsilon}{1-\lambda\gamma} \end{pmatrix}.
\end{eqnarray*}
Then
$$
\T^{\pi'}_\lambda v'-\T^\pi_\lambda v = \begin{pmatrix}\frac{1}{1-\lambda\gamma}-1 \\ \frac{1}{1-\lambda\gamma}-1 \end{pmatrix}
$$
while 
$$
v'-v = \begin{pmatrix}-\epsilon\\ \epsilon\end{pmatrix}.
$$
As $\epsilon$ can be arbitrarily small, the norm of $\T_\lambda^\pi v - \T_\lambda^{\pi'}v'$ can be arbitrarily larger than norm of $v-v'$ when $\lambda>0$. 
\end{proof}
Analyses of Policy Iteration usually rely on the fact that the
sequence of values generated by Exact Policy Iteration is non-decreasing (see
\cite{ndp,munosapi}). Unfortunately, it can easily be seen that as soon as $\lambda$ is smaller than 1, the value functions may decrease (it suffices to 
take a very high initial value). For non trivial values of $\lambda$, \lpi{ }is neither contracting nor non-decreasing, so we need a new proof technique.

\subsection{An Overview on the Componentwise Analysis of \lpi}

\label{proofoverview}

The rest of this section provides an overview of our analysis.
We show how to compute an upper bound of the loss for \lpi{ }in the general (possibly approximate) case.
It is the basis for the derivation of componentwise bounds for Exact \lpi{ }(Section~\ref{elpi}) and Approximate \lpi{ }(Section~\ref{alpi}). 
Consider \lpi{ }as described in Algorithm~\ref{algo:lpi}, and the sequences of value-policy-error triplets $(v_k,\pi_k, \epsilon_k)$ it generates. Most of our results come from a series of relations involving objects we now define:
\begin{itemize}
\item the {\bf loss} of using policy $\pi_k$ instead of the optimal policy:
$$
l_k:=v_*-v^{\pi_k};
$$
\item the {\bf value} of the $k^\text{th}$ iterate b.a. (before approximation):
$$
w_k:=v_k-\epsilon_k;
$$
\item the {\bf distance} between the optimal value and the $k^{th}$ value b.a.:
$$
d_k:=v_*-w_k=\T_\lambda^{\pi_k}v_{k-1};
$$
\item the {\bf \shift}{ }between the $k^\text{th}$ value b.a. and the value of the $k^{th}$ policy:
$$
\s_k:=w_k-v^{\pi_k};
$$
\item the {\bf \bell}{ }of the $k^\text{th}$ value:
$$
\b_k:=\T_{k+1} v_k -v_k=\T_* v_k - v_k.
$$
\end{itemize} 
To lighten the notations, we now on write: $P_k:=P^{\pi_k}$, $\T_k:=\T^{\pi_k}$, $P_*:=P^{\pi_*}$. We refer to the factor $\beta$ as introduced by Bertsekas and Ioffe (Equation~\ref{betadef} page \pageref{betadef}). Also, the following stochastic matrix plays a recurrent role in our analysis\footnote{\label{stocmat}The fact that this is indeed a stochastic matrix is explained at the beginning of the Appendices.}:
\begin{equation}
\label{defak}
A_k    := (1-\lambda\gamma)(I-\lambda\gamma P_k)^{-1}P_k.
\end{equation}
We use the notation $\overline{x}$ for an upper bound of $x$ and $\underline{x}$ for a lower bound.

Our analysis relies on a series of lemmas that we now state (for clarity, all the proofs are deferred to Appendix~\ref{appcore}).
{\begin{lemma}
\label{lbg}
The \shift{ }is related to the \bell{ }as follows:
$$
\s_k  =  \beta (I-\gamma P_k)^{-1}A_k(-\b_{k-1}).
$$
\end{lemma}}
{\begin{lemma}
\label{lrecg}
The \bell{ }at iteration $k+1$ cannot be much lower than that at iteration~$k$:
$$
\b_{k+1}  \geq  \beta A_{k+1}\b_k + x_{k+1} 
$$
where $x_k:=(\gamma P_k-I)\epsilon_k$ only depends on the approximation error.
\end{lemma}}
As a consequence, a lower bound of the \bell{ }is\footnote{\label{xypxpy}We use the property here that if some vectors satisfy the componentwise inequality $x \leq y$, and if $P$ is a stochastic matrix, then the componentwise inequality $Px \leq Py$ holds.}:
$$
\b_k \geq \sum_{j=k_0+1}^{k} \beta^{k-j}\left(A_{k}A_{k-1}...A_{j+1}\right) x_{j} + \beta^{k-k_0}\left(A_{k} A_{k-1}...A_{k_0+1}\right)\b_{k_0} := \underline{\b_k}
$$
where $k_0$ is some arbitrary reference index.
Using Lemma~\ref{lbg}, the bound on the \bell{ }also provides an upper on the \shift\footnote{We use the fact that $(1-\gamma)(I-\gamma P_k)^{-1}$ is a stochastic matrix (see Footnote~\ref{stocmat}) and Footnote~\ref{xypxpy}.}:
$$
\s_k \leq \beta (I-\gamma P_k)^{-1} A_k(-\underline{\b_{k-1}}) := \overline{\s_k}
$$
{\begin{lemma}
\label{lrecd}
The distance at iteration $k+1$ cannot be much greater than that at iteration~$k$:
$$
d_{k+1}  \leq  \gamma P_* d_k + y_k \\
$$
where $y_k:=\frac{\lambda\gamma}{1-\lambda\gamma} A_{k+1}(-\underline{\b_{k}})-\gamma P_*\epsilon_k$ depends on the lower bound of the \bell{ }and the approximation error.
\end{lemma}}
Then, an upper bound of the distance is\footnote{See Footnote~\ref{xypxpy}.}:
$$
d_k \leq \sum_{j=k_0}^{k-1} \gamma^{k-1-j}(P_*)^{k-1-j} y_{j} + \gamma^{k-k_0}(P_*)^{k-k_0}d_{k_0} = \overline{d_k}.
$$
Eventually, as 
$$
l_k = d_k+\s_k \leq \overline{d_k} + \overline{\s_k},
$$
the upper bounds on the distance and the \shift{ }enable us to derive the upper bound on the loss.

\begin{remark}
\label{inspiredbymunos}
The above derivation is a generalization of that of 
\cite{munosapi} for Approximate Policy Iteration. Note however that it is not a trivial generalization: when $\lambda=1$,
that is when both proofs coincide, $\beta=0$ and Lemmas
\ref{lbg} and~\ref{lrecg} have the following particularly simple form: $s_k=0$ and $b_{k+1} \geq x_{k+1}$.
\end{remark}

The next two sections contain our main results, which take the form of
performance bounds when using \lpi. Section~\ref{elpi} gathers
the results concerning Exact \lpi, while Section~\ref{alpi} presents those
concerning Approximate \lpi.


\section{Performance Bounds for Exact \lpi}

\label{elpi}

Consider Exact \lpi{ }for which we have $\epsilon_k=0$ for all $k$. Let $k_0$ be some arbitrary index. By exploiting the recursive relations we have described in the previous section (this process is detailed in Appendix~\ref{appelpi}), we can derive the following componentwise bounds for the loss:
{
\begin{lemma}[Componentwise Rate of Convergence of Exact $\lambda$PI]
\label{thexact}~\\
For all $k>k_0$, the following matrices
{\begin{eqnarray*}
E_{kk_0}  & := &  (1-\gamma)(P_*)^{k-k_0}(I-\gamma P_*)^{-1}, \\
E'_{kk_0}& :=& \left(\frac{1-\gamma}{\gamma^{k-k_0}}\right)\left(\frac{\lambda \gamma }{1-\lambda\gamma} \sum_{j=k_0}^{k-1}\gamma^{k-1-j}\beta^{j-k_0} (P_*)^{k-1-j}A_{j+1}A_{j}...A_{k_0+1} \right.\\
& & \hspace{5cm} \left. + \phantom{ \sum_{j=k_0}^{k-1}} \beta^{k-k_0}(I-\gamma P_k)^{-1}A_k A_{k-1}...A_{k_0+1}\right), \\
\mbox{and~~}F_{kk_0} & := &(1-\gamma) P_*^{k-k_0}+\gamma E'_{kk_0} P_*
\end{eqnarray*}}
are stochastic and the performance of the policies generated by \lpi{ }satisfies
\begin{eqnarray}
v_*-v^{\pi_k}  &\leq& \frac{\gamma^{k-k_0}}{1-\gamma}\left[ F_{kk_0}-E'_{kk_0} \right] (v_*-v_{k_0}), \label{thexacteq1}\\
v_*-v^{\pi_k}  &\leq &\frac{\gamma^{k-k_0}}{1-\gamma}\left[ E_{kk_0}-E'_{kk_0} \right] (\T_* v_{k_0}-v_{k_0}), \label{thexacteq2} \mbox{~~and}\\
v_*-v^{\pi_k} &\leq& \gamma^{k-k_0}\left[(P_*)^{k-k_0}\left((v_*-v_{k_0})-\min_s[v_*(s)-v_{k_0}(s)]e\right) + \norm{v_*(s)-v^{\pi_{k_0+1}}}{\infty} \e\right].~~~\label{thexacteq3}
\end{eqnarray}
\end{lemma}
}
In order to derive (more interpretable) span seminorms bounds from the above componentwise bound, we rely on the following lemma, which for clarity of exposition is proved in Appendix~\ref{fourlemmas}.
{
\begin{lemma}
\label{compnorm0}
If for some non-negative vectors $x$ and $y$, some constant $K \ge 0$, and some stochastic matrices $X$ and $X'$ we have
$$
x \leq K (X-X')y,
$$
Then
$$
\norm{x}{\infty} \leq K \spn{y}{\infty}.
$$
\end{lemma}
With this, the componentwise bounds of Lemma~\ref{thexact} become:
\begin{proposition}[Non-asymptotic bounds for Exact \lpi]~\\
For any $k>k_0$,
\begin{eqnarray}
\norm{v_*-v^{\pi_k}}{\infty} &\leq & \frac{\gamma^{k-k_0}}{1-\gamma} \spn{v_*-v_{k_0}}{\infty},\label{newbound0}\\
\norm{v_*-v^{\pi_k}}{\infty} &\leq & \frac{\gamma^{k-k_0}}{1-\gamma} \spn{\T_* v_{k_0}-v_{k_0}}{\infty},\label{newbound1}\\
\mbox{and~~}\norm{v_*-v^{\pi_k}}{\infty} &\leq &\gamma^{k-k_0} \left( \spn{v_*-v_{k_0}}{\infty}+\norm{v_*(s)-v^{\pi_{k_0+1}}}{\infty}\right)  \label{newbound}.
\end{eqnarray}
\end{proposition}
This \emph{non-asympotic} bound supplements the \emph{asymptotic} bound of Proposition~\ref{lpiorig} from \citet{ioffe}. Remarkably, these bounds do not depend on the value $\lambda$: whatever the value of $\lambda$, all algorithms have the same above rates. The bound of Equation~\ref{newbound0} is expressed in terms of the distance between the value function and the optimal value function at some iteration $k_0$. The second inequality, Equation~\ref{newbound1}, can be used as a stopping criterion. Indeed, taking $k=k_0+1$ it implies the following stopping condition, which generalizes that of Proposition~\ref{stopvi} about Value Iteration:
{\begin{proposition}[Stopping condition of Exact \lpi]
\label{stopexact}~\\
If at some iteration $k_0$, the value $v_{k_0}$ satisfies: 
$$
\spn{\T_* v_{k_0}-v_{k_0}}{\infty} \leq \frac{1-\gamma}{\gamma}\epsilon
$$ 
then the greedy policy $\pi_{k_0+1}$  with respect to $v_{k_0}$ is $\epsilon$-optimal: 
$\norm{v_*-v^{\pi_{k_0+1}}}{\infty} < \epsilon$.
\end{proposition}}
The last inequality described in Equation~\ref{newbound} relies on the distance between the value
function and the optimal value function and the value difference
between the optimal policy and the first greedy policy; compared to
the others, it has the advantage of not containing  a
$\frac{1}{1-\gamma}$ factor. To our knowledge, this bound is
even new for the specific cases of Value Iteration and Policy
Iteration.



\section{Performance Bounds for Approximate \lpi}

\label{alpi}

We now turn to the (slightly more involved) results on Approximate \lpi. 
We provide componentwise bounds of the loss $v_*-v^{\pi_k} \geq 0$ of using policy $\pi_k$ instead of using the optimal policy, with respect to the approximation error $\epsilon_k$, the \pbell{ }$\T_k v_k - v_k$ and the \bell{ }$\T_* v_k - v_k=\T_{k+1}v_k - v_k$. Recall the subtle difference between these two \bell s: the \pbell{ }says how much $v_k$ differs from the value of $\pi_k$ while the \bell{ }says how much $v_k$ differs from the value of the policies $\pi_{k+1}$ and $\pi_*$.

The core of our analysis is based on the following lemma:
{\begin{lemma}[Componentwise Performance bounds for App. \lpi]
\label{th}~\\
For all $k>j\ge 0$, the following matrices
{\begin{eqnarray*}
B_{jk}& := & \frac{1-\gamma}{\gamma^{k-j}}\left[ \frac{\lambda\gamma}{1-\lambda\gamma} \sum_{i=j}^{k-1}\gamma^{k-1-i}\beta^{i-j}(P_*)^{k-1-i}A_{i+1}A_{i}...A_{j+1} \right. \\
& & \hspace{5cm} \left. \phantom{\sum_{i=j}^{k-1}\gamma^{k-1-i}}+ \beta^{k-j}(I-\gamma P_k)^{-1}A_{k}A_{k-1}...A_{j+1} \right] \\
B'_{jk}& := & \gamma B_{jk} P_j + (1-\gamma)(P_*)^{k-j} \\
C_k & := & (1-\gamma)^2 (I-\gamma P_*)^{-1}\left( P_*(I-\gamma P_k)^{-1} \right) \\
C'_k & := & (1-\gamma)^2 (I-\gamma P_*)^{-1}\left(  P_{k+1}(I-\gamma P_{k+1})^{-1} \right) \\
D & := & (1-\gamma)P_*(I-\gamma P_*)^{-1} \\
D'_k & := & (1-\gamma)P_{k}(I-\gamma P_{k})^{-1}
\end{eqnarray*}}
are stochastic and
\begin{eqnarray}
\forall k_0,\mbox{~~} \ls{k}v_*-v^{\pi_k} & \leq & \frac{1}{1-\gamma} \ls{k}\sum_{j=k_0}^{k-1} \gamma^{k-j}\left[ B_{jk} - B'_{jk}\right] \epsilon_j \label{ccbbp0}, \\
\ls{k}v_*-v^{\pi_k} & \leq & \frac{\gamma}{(1-\gamma)^2}\ls{k} [C_k - C'_k] (\T_k v_k - v_k), \label{ccbbp1}  \\
\mbox{and~~}\forall k,\mbox{~~} v_*-v^{\pi_k} &\leq &\frac{\gamma}{1-\gamma}\left[ D - D'_k \right](\T_* v_{k-1} - v_{k-1}).\label{ccbbp2} 
\end{eqnarray}
\end{lemma}}
The first relation (Equation~\ref{ccbbp0}) involves the errors $(\epsilon_k$), is based on Lemmas~\ref{lbg}-\ref{lrecd} (presented in Section~\ref{resultsandproofidea}) and is proved in Appendix~\ref{appalpi}. The two other inequalities (the asymptotic performance of Approximate \lpi{ }with respect to the Bellman residuals in Equations~\ref{ccbbp1} and \ref{ccbbp2}) are somewhat simpler and are
proved independently in Appendix~\ref{appbell}.

\begin{remark}[Relation with the previous bounds of \cite{munosavi,munosapi}]
We can look at the relation between our bound for general $\lambda$ and the bounds  derived separately by Munos for Approximate Value Iteraton (Lemma~\ref{cavi}) and Approximate Policy Iteraton (Lemma~\ref{capi}):
\begin{itemize}
\item Let us first consider the case where $\lambda=0$. Then $\beta=\gamma$, $A_k=P_{k}$ and 
$$
B_{jk}=(1-\gamma)(I-\gamma P_k)^{-1}P_k P_{k-1}...P_{j+1}.
$$
Then our bound implies that $\ls{k} v_*-v^{\pi_k}$ is upper bounded by:
{
\begin{eqnarray}
& &\ls{k}  \sum_{j=k_0}^{k-1} \gamma^{k-j} \left[ (I-\gamma P_k)^{-1}P_k P_{k-1}...P_{j+1}  \phantom{\left( (P_*)^{k-j} \right)} \right.\nonumber \\
& & \hspace{4cm}- \left. \left(\gamma(I-\gamma P_k)^{-1} P_k P_{k-1}...P_{j}+(P_*)^{k-j}\right)\right]\epsilon_j.
\label{avi2}
\end{eqnarray}}
The bound derived by Munos for Approximate Value Iteration (Lemma~\ref{cavi} page \pageref{cavi}) is
{
\begin{eqnarray}
& & \hspace{-.3cm}\ls{k} (I-\gamma P_k)^{-1} \sum_{j=0}^{k-1} \gamma^{k-j} \left[ P_k P_{k-1}...P_{j+1} -  (P_*)^{k-j} \right]\epsilon_j \nonumber\\
& \hspace{-.3cm} = & \hspace{-.3cm}\ls{k} \sum_{j=0}^{k-1} \gamma^{k-j} \left[ (I-\gamma P_k)^{-1} P_k P_{k-1}...P_{j+1} -  (I-\gamma P_k)^{-1} (P_*)^{k-j} \right]\epsilon_j \nonumber\\
& \hspace{-.3cm}= & \hspace{-.3cm} \ls{k} \sum_{j=0}^{k-1} \gamma^{k-j} \left[ (I-\gamma P_k)^{-1} P_k P_{k-1}...P_{j+1}  \phantom{\left( (P_*)^{k-j} \right)}  \right. \nonumber \\
& & \hspace{4cm}- \left. \left((I-\gamma P_k)^{-1}\gamma P_k(P_*)^{k-j}+(P_*)^{k-j}\right) \right]\epsilon_j. \label{avimunos2}
\end{eqnarray}
}
The above bounds are very close to each other: we can go from Equation~\ref{avi2} to Equation~\ref{avimunos2} by replacing $P_{k-1}...P_{j}$ by $(P_*)^{k-j}$.
\item Now, when $\lambda=1$, $\beta=0$, $A_k=(1-\gamma)(I-\gamma P_k)^{-1} P_k$ and
$$
B_{jk}= (1-\gamma)(P_*)^{k-1-j} P_{j+1}(I-\gamma P_{j+1})^{-1}.
$$
Our bound is
$$
\ls{k} v_*-v^{\pi_k} \leq  \ls{k}\sum_{j=k_0}^{k-1} \gamma^{k-1-j}(P_*)^{k-1-j} u_j
$$
with 
$$
u_j:= \left[\gamma P_{j+1}(I-\gamma P_{j+1})^{-1}(I-\gamma P_j)-\gamma P_* \right]\epsilon_j.
$$
By definition of the supremum limit, for all $\epsilon>0$, there exists an index $k_1$ such that for all $j \geq k_1$, 
$$
u_j \leq \ls{l} u_l + \epsilon \e.
$$
Then:
\begin{eqnarray*}
\ls{k}\sum_{j=k_1}^{k-1} \gamma^{k-1-j}(P_*)^{k-1-j} u_j & \leq &  \ls{k}\sum_{j=k_1}^{k-1} \gamma^{k-1-j}(P_*)^{k-1-j} \left( \ls{l} u_l + \epsilon\e \right)\\
& = & (I-\gamma P_*)^{-1}\left( \ls{l} u_l + \epsilon \e\right).
\end{eqnarray*}
As this is true for all $\epsilon>0$, we eventually find the bound of Munos for Approximate Policy Iteration (Lemma~\ref{capi} page \pageref{capi}).
\end{itemize}
Thus, up to some little details, our componentwise analysis unifies
those of Munos. It is not a surprise that we fall back on the result
of Munos for Approximate Policy Iteration because, as already
mentionned in Remark~\ref{inspiredbymunos}, the proof we developed in
Section~\ref{resultsandproofidea} and Appendix~\ref{appcore} is a
generalization of his. If we don't exactly recover the componentwise
analysis of Munos for Approximate Value Iteration, this is not really
fundamental as it will not affect most of the results we derive.
\end{remark}

The componentwise bounds on the performance of Approximate \lpi{ }can be translated into span seminorm bounds, using the following Lemma (proved in Appendix~\ref{fourlemmas}):
\begin{lemma}
\label{compnorm}
Let $x_k$, $y_k$ be vectors and $X_{jk}$, $X'_{jk}$ stochastic matrices satisfying
$$
\forall k_0,\mbox{~~}\ls{k}|x_k| \leq K \ls{k}\sum_{j=k_0}^{k-1} \xi_{k-j} (X_{kj}-X'_{kj})y_j,
$$
where $(\xi_i)_{i \ge 1}$ is a sequence of non-negative weights satisfying:
$$
\sum_{i=1}^{\infty} \xi_i = K' < \infty,
$$
then, for all distribution $\mu$, 
$$
{\mu_{kj}}:=\frac{1}{2} \transp{ (X_{kj}+X'_{kj})}\mu
$$
are distributions and
$$
\ls{k}\norm{x_k}{p,\mu} \leq K K' \lim_{k_0 \rightarrow \infty} \left[\sup_{k \geq j \geq k_0}\spn{y_j}{p,\mu_{kj}}\right].
$$
\end{lemma}
Thus, using this Lemma and the fact that $\sum_{i=1}^\infty \gamma^i=\frac{\gamma}{1-\gamma}$, Lemma~\ref{th} can be turned into the following proposition that
unifies and generalizes Proposition~\ref{lpavi} (page \pageref{lpavi}) on Approximate Value Iteration  and Proposition~\ref{lpapi} (page \pageref{lpapi}) on Approximate Policy Iteration.
{\begin{proposition}[Span Seminorm Performance of Approximate $\lambda$PI (1/2)]
\label{spalpi}~\\
With the notations of Lemma~\ref{th}, for all $p$, $k > j \ge 0$ and all distribution $\mu$,
\begin{align}
\mu_{kj}&:=\frac{1}{2}\transp{\left({B_{jk}+B'_{jk}}\right)}\mu, \nonumber \\
\mu'_{kj}&:=\frac{1}{2}\transp{\left({C_{k}+C'_{k}}\right)}\mu, \mbox{~and} \nonumber\\
\mu''_{kj}&:= \frac{1}{2}\transp{\left({D+D'_{k}}\right)}\mu \nonumber
\end{align}
are distributions and the performance of the policies generated by \lpi{ }satisfies:
\begin{eqnarray*}
\ls{k} \norm{v_*-v^{\pi_k}}{p,\mu}&  \leq& \frac{\gamma}{(1-\gamma)^2} \lim_{k_0 \rightarrow \infty} \left[\sup_{k \geq j \geq k_0}\spn{\epsilon_j}{p,\mu_{kj}}\right], \\
\ls{k} \norm{v_*-v^{\pi_k}}{p,\mu}&  \leq& \frac{\gamma}{(1-\gamma)^2} \ls{k} \spn{\T_k v_k-v_k}{p,\mu'_{kj}}, \\
\forall k,\mbox{~~} \norm{v_*-v^{\pi_k}}{p,\mu} &\leq & \frac{\gamma}{1-\gamma} \spn{\T_* v_{k-1}-v_{k-1}}{p,\mu''_{kj}}.
\end{eqnarray*}
\end{proposition}}

As already mentionned, a drawback of the above \lp{ }bounds comes from the fact that the distributions involved on the right hand sides are unknown. To go round this issue, one may consider the {\bf concentration coefficient} assumption introduced by \cite{munosapi,munosavi} and already mentioned in Equation~\ref{conccoef} page \pageref{conccoef}. For clarity, we recall its definition here. We assume there exists a distribution $\nu$ and a real number $C(\nu)$ such that
$$
C(\nu):=\max_{i,j,a} \frac{p_{ij}(a)}{\nu(j)}.
$$
Then, we have the following property:
\begin{lemma}
\label{fromcomptospan}
Let $X$ be an average of products of stochastic matrices of the MDP. For any distribution $\mu$, and vector $y$ and any $p$,
$$
\spn{y}{p,\transp{X}\mu} \leq \left(C(\nu)\right)^{1/p} \spn{y}{p,\nu}.
$$
\end{lemma}
\begin{proof}
It can be seen from the definition of the concentration coefficient $C(\nu)$ that
$\transp{\mu} X   \leq  C(\nu)\transp{\nu}$.
Thus, 
\begin{eqnarray*}
\left(\spn{y}{p,\transp{X} \mu}\right)^p & =& \min_a \left(\norm{y-a\e}{p,\transp{X} \mu}\right)^p \\
& = & \min_a \transp{\mu} X |y-a\e|^p \\
& \leq & C(\nu) \min_a  \transp{\nu} |y-a\e|^p \\
& = &  C(\nu) \min_a \left(\norm{y-a\e}{p,\nu}\right)^p \\
& = & C(\nu) \left(\spn{y}{p,\nu}\right)^p. \qedhere
\end{eqnarray*}
~\vspace{-1.7cm}\\

\end{proof}

Using this Lemma, and the fact that for any $p$, $\norm{x}{\infty}=\max_\mu \norm{x}{p,\mu}$, the \lp{ }bounds of Proposition~\ref{spalpi} become
{\begin{proposition}[Span Seminorm  Performance of Approximate $\lambda$PI (2/2)]
\label{calpi}~\\
Let $C(\nu)$ be the concentration coefficient defined in Equation~\ref{conccoef}. For all $p$ and all $k$,
\begin{eqnarray*}
\ls{k} \norm{v_*-v^{\pi_k}}{\infty}&  \leq & \frac{\gamma}{(1-\gamma)^2}\left[C(\nu)\right]^{1/p} \ls{j}\spn{\epsilon_j}{p,\nu}, \\
\ls{k} \norm{v_*-v^{\pi_k}}{\infty}&  \leq& \frac{\gamma}{(1-\gamma)^2}\left[C(\nu)\right]^{1/p} \ls{k} \spn{\T_k v_k-v_k}{p,\nu}, \\
\mbox{and~~}\forall k,\mbox{~~} \norm{v_*-v^{\pi_k}}{\infty} &\leq & \frac{\gamma}{1-\gamma}\left[C(\nu)\right]^{1/p} \spn{\T_* v_{k-1}-v_{k-1}}{p,\nu}.
\end{eqnarray*}
\end{proposition}}
This results generalizes and unifies those derived for Approximate Value Iteration (Proposition~\ref{concavi} page \pageref{concavi}) and Approximate Policy Iteration (Proposition~\ref{concapi} page \pageref{concapi}). 

When comparing the specific bounds of Munos for Approximate Value
Iteration (Propositions~\ref{lpavi} and~\ref{concavi}) and Approximate
Policy Iteration (Propositions~\ref{lpapi} and~\ref{concapi}), we
wrote that the latter had the nice property that the bounds only
depend on \emph{asymptotic} errors/residuals (while the former depends
on all errors). Our bounds for \lpi{ }have this nice property too.
Considering the relations between the span seminorms and the other
standard norms (Equation~\ref{norms} page \pageref{norms}), we see
that our results are not only more general, but also slightly finer
than those of Munos.

\paragraph{When the policy or the value converges}

The  performance bounds with
respect to the approximation error can be improved if we know or
observe that the value or the policy converges. Note that the former
condition implies the latter (while the opposite is not true: the
policy may converge while the value still oscillates).
Indeed, we have the following Corollary.
\begin{corollary}[Performance of Approximate $\lambda$PI in case of convergence]
\label{valueorpolicyconvergence}
~\\
If the value converges to some $v$, then the approximation error converges to some $\epsilon$, and the corresponding greedy policy $\pi$ satisfies
\begin{eqnarray*}
\norm{v_*-v^{\pi}}{\infty} & \leq & \frac{\gamma}{1-\gamma}\left[C(\nu)\right]^{1/p} \spn{\epsilon}{p,\nu}.
\end{eqnarray*}
If the policy converges to some $\pi$, then
\begin{eqnarray*}
\norm{v_*-v^{\pi}}{\infty} & \leq & \frac{\gamma(1-\lambda\gamma)}{(1-\gamma)^2}\left[C(\nu)\right]^{1/p} \ls{j}\spn{\epsilon_j}{p,\nu}.
\end{eqnarray*}
\end{corollary}
These bounds, proved in Appendix~\ref{valuepolicyconv},
unify and extend those presented for Approximate Value Iteration
(Corollary~\ref{avibell} page \pageref{avibell}) and Approximate
Policy Iteration (Corollary~\ref{apibell} page \pageref{apibell}), in
the similar situation where the policy or the value converges.  It is
interesting to notice that in the weaker situation where only the
policy converges, the constant decreases from $\frac{1}{(1-\gamma)^2}$
to $\frac{1}{1-\gamma}$ when $\lambda$ varies from 0 to 1; in other
words, the closer to Policy Iteration, the better the bound in that
situation.


\section{Application of \lpi{ }to the Game of Tetris}

\label{thecasestudy}

In the final part of this paper, we consider (and describe for
the sake of keeping this paper self-contained) exactly the same
application (Tetris) and implementation as \cite{ioffe}.  Our motivation
for doing so is twofold:
\begin{itemize}
\item from a theoretical point of view, we show how our
analysis (made in the discounted case $\gamma<1$) can be
adapted to an undiscounted problem (where $\gamma=1$) like Tetris; 
\item  we obtain empirical results that are different (and much
less intriguing) than those of the original study. This gives us
the opportunity to describe what we think are the reasons for such a difference.
\end{itemize}
But before doing so, we begin by describing the Tetris domain.

\subsection{The Game of Tetris and its Model as an MDP}

Tetris is a popular video game created in 1985 by Alexey Pajitnov.
The game is played on a $10 \times 20$  grid 
where pieces of different shapes fall from the top (see Figure~\ref{tetris_overview}).
The player has to choose where each piece is added: he can move it
horizontally and rotate it. When a row is filled, it is removed and
all cells above it move one row downwards. The goal is to remove as many 
lines as possible before the game is over, that is when there is
not enough space remaining on the top of the pile to put the current new
piece. 
\begin{figure}[t]
  \begin{minipage}[l]{.49\linewidth}
    \begin{center}
      \includegraphics[width=7cm]{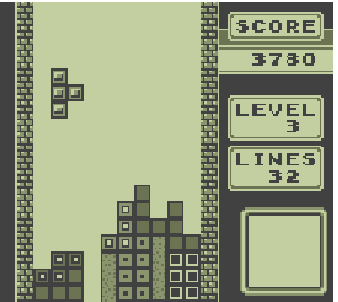} 
    \end{center}
  \end{minipage}
  \begin{minipage}[l]{.49\linewidth}
    \begin{center}
      \includegraphics[width=7cm]{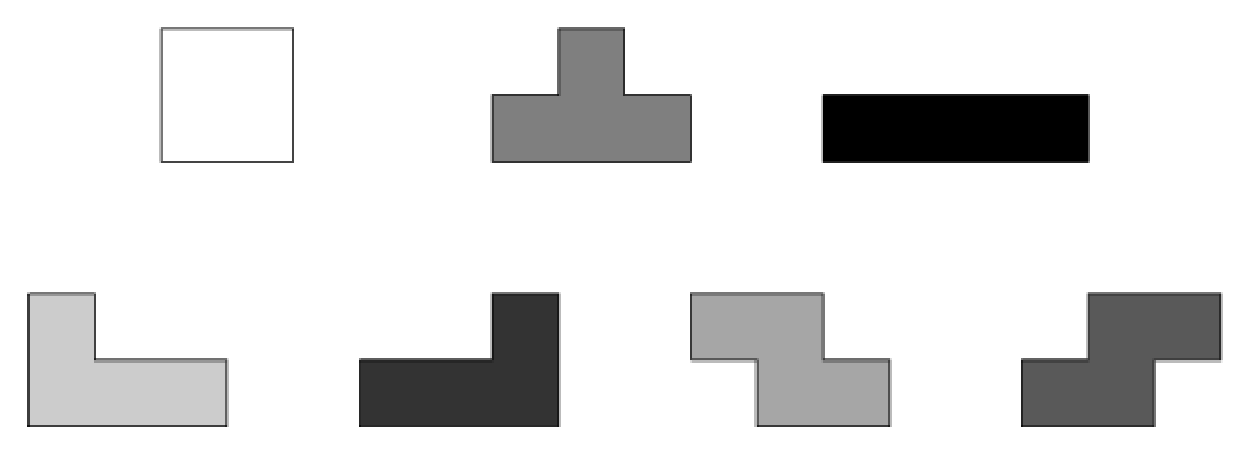} 
    \end{center}
  \end{minipage}
  \caption{\label{tetris_overview}{\bf Left:} a screenshot of a Tetris game. {\bf Right:} the seven existing shapes.}  
\end{figure}

Instead of mimicking the original game (precisely described by
\citet{Fahey}), \citet{ioffe} have focused on the main problem, that
is choosing where to drop each coming piece.  The corresponding MDP
model is straightforward: the state consists of the wall configuration
and the shape of the current falling piece. An action is the
horizontal translation and the rotation which are applied to the piece
before it is dropped on the wall. The reward is the number of lines
which are removed after we have dropped the piece.  As one considers
the maximization of the score (the total number of lines removed
during a game), the natural choice for the  discount factor is
$\gamma=1$.


In a bit more details, the dynamics of Tetris is made of two
components: the place where one drops the current piece and the choice
of a new piece. As the latter component is uncontrollable (a new piece
is chosen with uniform probability), the value functions needs not to
be computed for all wall-piece pairs configurations but only for all
wall configurations (see for instance \citep{ioffe}). Also considering
that the first component of the dynamics is deterministic, the optimal
value function satisfies a reduced form of the Bellman Equation
\begin{equation}
\label{bellmantetris}
\forall s \in S, v_*(s)= \frac{1}{|{\cal P}|}\sum_{p \in {\cal P}}  \max_{a \in A(p)} r(s,p,a) + v_*(succ(s,p,a))
\end{equation}
where $S$ is the set of wall configurations, ${\cal P}$ is the set
of pieces, $A(p)$ is the set of translation-rotation pairs that can be
applied to a piece $p$, $r(s,p,a)$ and $succ(s,p,a)$ are respectively
the number of lines removed and the (deterministic) next wall configuration if one
puts a piece $p$ on the wall $s$ in translation-orientation $a$. The
only function that satisfies the above Equation gives, for each wall
configuration $s$, the average best score that can be achieved from
$s$. If we know this function, a one step look-ahead strategy (that is a
greedy policy) performs optimally.

\paragraph{Extension of the analysis for the undiscounted optimal control problem Tetris}

As just explained, the Tetris domain has a discount factor
$\gamma$ equal to $1$, which makes it an \emph{undiscounted MDP}.  If this prevents
us from applying directly most of the analysis we have made so
far (since most of our bounds have a $(1-\gamma)$ term on the
  denominator), we briefly show in what follows how to 
adapt the analysis so that we recover a significant error analysis.

In undiscounted infinite horizon control problems, it is generally assumed that there exists a ${N+1}^{th}$ termination absorbing
state $0$.  Once the system reaches this state, it remains there forever with no further reward, that is formally:
$$
\forall a, p_{0 0}(a)=1 \mbox{ and }r(0,a,0)=0.
$$
In the case of Tetris, the termination state corresponds to ``game
over'', and the situation is particulary simple since \citet{burgiel}
has shown that, whatever the strategy, some sequence of pieces (which
necessarily occurs in finite time with probability 1) leads to
game-over whatever the decisions taken\footnote{In the literature, a
  stationary policy that reaches the terminal state in finite time
  with probability 1 is said to be \emph{proper}. The usual
  assumptions in undiscounted infinite horizon control problems are: \emph{(i)}
  there exists at least one proper policy and \emph{(ii)} for every improper
  policy $\pi$, the corresponding value equals $-\infty$ for at least
  one state. A simpler situation is when all stationary policies are
  proper. The situation of Tetris is even simpler: all \emph{non
  necessarily stationary} policies are proper.}. This implies in
particular that there exists an integer $n_0 \leq N$ and a real number
$\alpha<1$ such that for all initial distributions $\mu$, and
actions $a_0,a_1,...,a_{n_0-1}$,
\begin{equation}
\label{assumption}
P\left[ i_{n_0} \neq 0 | i_0 \sim \mu, a_0, ..., a_{n_0-1}\right] \leq \alpha.
\end{equation}
We can define the MDP model for Tetris only on the $N$ non-terminal
states, that is on $\{1,...N\}$. In this situation, for any policy
$\pi$, the matrix $P_\pi$ is in general a \emph{substochastic} matrix (a matrix of which the components are non-negative and for which the max norm is smaller than or equal to $1$),
and the above assumption means that for all set of $n_0$ policies
$\pi_1,\pi_2,\cdots, \pi_{n_0}$,
\begin{equation*}
\norm{P_{\pi_1} P_{\pi_2} \cdots  P_{\pi_{n_0}}}\infty \leq \alpha.
\end{equation*}

The componentwise analysis of \lpi{ }is here identical to what we have done before, except that we have\footnote{For simplicity in our discussion,  we consider $\lambda<1$ to avoid the special case $\lambda=1$ for which $\beta=0$ (see Equation~\ref{betadef}). The interested reader may however check that the results that we state are continuous in the neighborhood of $\lambda=1$.} $\gamma=1$ and $\beta=1$.
The matrix $A_k$ that appeared recurrently in our analysis has the following special form:
$$
A_k :=(1-\lambda)(I-\lambda P_k)^{-1} P_k. 
$$
and is a substochastic matrix.
The first bound of the componentwise analysis of \lpi{ }(Lemma~\ref{th} page \pageref{th}) can be shown to be generalized as follows (see Appendix~\ref{proofundiscounted} for details): 
\begin{lemma}[Componentwise Bounds in the Undiscounted Case]
~\label{lemmaundiscounted}\\
Assume that there exists $n_0$ and $\alpha$ such that Equation~\ref{assumption} holds. Write $\eta:=\frac{1-\lambda^{n_0}}{1-\lambda^{n_0}\alpha}$. For all $i$, write
$$
\delta_{i}:={ \alpha^{\left \lfloor \frac{i}{n_0}\right\rfloor}}\left[\left(\frac{1-\lambda^{n_0}}{1-\lambda} \right)  \left(\frac{\lambda}{1-\lambda^{n_0}\alpha} \right)\left(\frac{1-\eta^{i}}{1-\eta}\right) + \frac{n_0\eta^{i}}{1-\alpha} \right].
$$
For all $j<k$, the following matrices
\begin{eqnarray*}
{G}_{jk}& := & \frac{1}{\delta_{k-j}}\left[ \frac{\lambda}{1-\lambda} \sum_{i=j}^{k-1} (P_*)^{k-1-i}A_{i+1}A_{i}...A_{j+1}+ (I-P_k)^{-1}A_k A_{k-1}...A_{j+1} \right]\\
\mbox{and~~}{G'}_{jk}& := & \frac{1}{\delta_{k-j}} G_{jk} P_j
\end{eqnarray*}
are substochastic and
the performance of the policies generated by \lpi{ }satisfies
\begin{equation}
\label{errorfortetris}
\forall k_0, \mbox{~~}\ls{k}v_*-v^{\pi_k} \leq  \ls{k}\sum_{j=k_0}^{k-1} \delta_{k-j} \left[ {G}_{jk} - {G'}_{jk}\right] \epsilon_j.
\end{equation}
\end{lemma}
\begin{remark}
By observing that $\eta \in (0,1)$, and that for all $x \in (0,1)$, $0 \leq \frac{1-x^{n_0}}{1-x} \leq n_0$, it can be seen that the coefficients $\delta_i$ are finite for all $i$. Furthermore, when $n_0=1$ (which matches the discounted case with $\alpha=\gamma$), one can observe that $\delta_i=\frac{\gamma^i}{1-\gamma}$ and that one recovers the result of Lemma~\ref{th}.
\end{remark}

This lemma can then be exploited to show that \lpi{ } enjoys an \lp{ }norm guarantee. Indeed, an analogue of Proposition~\ref{spalpi} (whose proof is detailed in Appendix~\ref{proofundiscounted}) is the following proposition.
\begin{proposition}[\lp{ }norm Bound for in the Undiscounted Case]
\label{boundundiscounted}~\\
Let $C(\nu)$ be the concentration coefficient defined in Equation~\ref{conccoef}. Let the notations of Lemma~\ref{lemmaundiscounted} hold. 
For all distribution $\mu$ on $(1,\cdots,N)$ and $k>j\ge 0$,
$$\mu_{kj}:=\frac{1}{2}\transp{\left({{G}_{jk}+{G}'_{jk}}\right)}\mu$$ 
are non-negative vectors and 
$$\tilde \mu_{kj}:=\frac{\mu_{kj}}{\norm{\mu_{kj}}{1}}$$ are distributions on $(1,\cdots,N)$.
Then for all $p$, the loss of the policies generated by \lpi{ }satisfies
\begin{equation*}
\ls{k} \norm{v_*-v^{\pi_k}}{p,\mu}  \leq  K(\lambda,n_0) (C(\nu))^{1/p} \lim_{j \rightarrow \infty} \norm{\epsilon_j}{p,\tilde{\mu}_{kj}}
\end{equation*}
where 
\begin{align}
K(\lambda,n_0)&:= \lambda {f}(\lambda)\frac{f(1) - {f}(\eta)}{1-\eta} + f(1) {f}(\eta)-f(1), \nonumber \\
\forall x<1,~f(x)&:=\frac{(1-x^{n_0})}{(1-x)(1-x^{n_0}\alpha)},\mbox{~and by continuity~}f(1):=\frac{n_0}{1-\alpha}.\nonumber
\end{align}
\end{proposition}
\begin{remark}
\label{remarkundiscounted}
There are three differences with respect to the results we have presented for the discounted case.
\begin{enumerate}
\item The fact that we defined the model (and thus the algorithm) only on the non-terminal states $(1,\cdots,N)$ implies that there is no error incurred in the terminal state $0$. Note, however, that this is not a strong assumption since the value of the terminal state is necessarily $0$.
\item The right hand side depends on the \lp{ }norm, and not the span \lp{ }seminorm. This is due to the fact that the matrices $G_{jk}$ and $G'_{jk}$ defined above are in general \emph{substochastic} matrices (and \emph{not} stochastic matrices).
\item Eventually, the constant $K(\lambda,n_0)$ depends on $\lambda$. More precisely, it can be observed that:
$$
\lim_{\lambda \rightarrow 0}K(\lambda,n_0)=\lim_{\lambda \rightarrow 1}K(\lambda,n_0)=\frac{{n_0}^2}{(1-\alpha)^2}-\frac{n_0}{1-\alpha}
$$
and that this is the minimal value of $\lambda \mapsto K(\lambda,n_0)$. Though we took particular care in deriving this bound, we leave for future work the question whether one could prove a similar result with the constant $\frac{{n_0}^2}{(1-\alpha)^2}-\frac{n_0}{1-\alpha}$ for any $\lambda \in (0,1)$.
When $n_0=1$ (which matches the discounted case with $\alpha=\gamma$), $K(\lambda,1)$ does not depend anymore on $\lambda$ and we recover (without surprise) the bound of Proposition~\ref{spalpi}:
$$
\forall \lambda, K(\lambda,1)=\frac{\alpha}{(1-\alpha)^2}.
$$
\end{enumerate}
\end{remark}

Now that we are reassured about the fact that applying \lpi{ }approximately to Tetris is principled, 
we turn to the precise description of its actual implementation.

\subsection{An Instance of Approximate \lpi}
\label{thealg}

For large scale problems, many Approximate Dynamic Programming algorithms are based on two complementary tricks:
\begin{itemize}
\item one uses samples to approximate the expectations such as that of Equation~\ref{tempdif};
\item  one only looks for a linear approximation of the optimal value
function: $$
v^\theta(s)=\theta(0) + \sum_{k=1}^K \theta(k) \Phi_k(s)
$$
where $\theta=(\theta(0) \dots \theta(K))$ is the parameter vector and $\Phi_k(s)$ are
some predefined feature functions on
the state space. 
Thus, each value of $\theta$ characterizes a value function $v^\theta$ over the entire
state space. 
\end{itemize}
The instance of Approximate \lpi{ }of \cite{ioffe} follows these
ideas. More specifically, this algorithm is devoted to MDPs which have
a termination state, that has 0 reward and is absorbing. For this
algorithm to be run, one must further assume that all policies are
proper, which means that all policies reach the termination state with
probability one in finite time\footnote{\citet{ioffe} consider a
  weaker assumption for Exact \lpi{ }and its analysis, namely that
  there exists at least \emph{one} propoer policy. However, this
  assumption is not sufficient for their Approximate algorithm,
  because this builds sample trajectories that need to reach a
  termination state. If the terminal state were not reachable in
  finite time, this algorithm may not terminate in finite time.}.

Similarly to Exact \lpi, this Approximate \lpi{ }maintains a \emph{compact} value-policy pair $(\theta_t,\pi_t)$. 
Given $\theta_t$, $\pi_{t+1}$ is the greedy policy with respect to $v^{\theta_t}$, and can easily be computed exactly in any given state as the argmax in Equation~\ref{bellmantetris}.
This policy $\pi_{t+1}$ is used to simulate a batch of $M$ trajectories: for each trajectory $m$, $(s_{m,0}, s_{m,1}, \dots, s_{m,N_m-1},s_{m,N_m})$ denotes the sequence of states of the $m^\text{th}$ trajectory, with $s_{m,N_m}$ being the termination state. Then for approximating Equation~\ref{tempdif}, a reasonable choice for $\theta_{t+1}$ is one that satisfies:
\begin{eqnarray}
v^{ \theta_{t+1}}(s_{m,N_m}) & \simeq &  0 \label{pb1} \\
v^{ \theta_{t+1}}(s_{m,N_m-1}) & \simeq &  v^{\theta_t}(s_{m,N_m-1}) + \td_t(s_{m,N_m-1},s_{m,N_m}) \nonumber \\
v^{ \theta_{t+1}}(s_{m,N_m-2}) & \simeq &  v^{\theta_t}(s_{m,N_m-2}) + \td_t(s_{m,N_m-2},s_{m,N_m-1}) + \gamma\lambda \td_t(s_{m,N_m-1},s_{m,N_m}) \nonumber \\
\vdots & & \vdots \nonumber \\
v^{ \theta_{t+1}}(s_{m,k}) & \simeq & v^{\theta_t}(s_{m,k}) + \sum_{s=k}^{N_{m-1}}(\gamma\lambda)^{s-k} \td_t(s_{m,s},s_{m,s+1}) \nonumber \\
\vdots & & \vdots \nonumber \\
v^{ \theta_{t+1}}(s_{m,0}) & \simeq & v^{\theta_t}(s_{m,0}) + \sum_{s=0}^{N_{m-1}}(\gamma\lambda)^s \td_t(s_{m,s},s_{m,s+1}) \nonumber
\end{eqnarray}
for all trajectories $m$, where
\begin{equation}
\td_t(s_{m,N_m-1},s_{m,N_m})=r(s_{m,N_m-1},\pi_{t+1}(s_{m,N_m-1}),s_{m,N_m}) - v^{\theta_t}(s_{m,N_m-1}) \label{pb2}
\end{equation}
and for all $s<N_m-1$
$$
\td_t(s_{m,s},s_{m,s+1})=r(s_{m,s},\pi_{t+1}(s_{m,s}),s_{m,s+1}) + \gamma v^{\theta_t}(s_{m,s+1}) - v^{\theta_t}(s_{m,s})
$$
are the temporal differences. Note  that Equations~\ref{pb1} and~\ref{pb2} correspond to the terminal states for which there is no subsequent rewards.
A standard and efficient solution to this problem consists
in minimizing the least squares error, that is to choose $\theta_{t+1}$ as follows:
$$
\theta_{t+1} = \arg\min_{\bf \theta} \sum_{m=1}^{M}\sum_{k=0}^{N_m} \left( v^{\bf \theta}(s_{m,k})- v^{\theta_t}(s_{m,k}) - \sum_{j=k}^{N_{m-1}}(\gamma\lambda)^{j-k} \td_t(s_{m,j},s_{m,j+1})\right)^2.
$$
This approximate version of \lpi{ }generalizes well-known algorithms. When $\lambda=0$, the generic term  becomes a sample of $[\T^{\pi_{k+1}} v](s_{m,k})$:
\begin{align}
v^{ \theta_{t+1}}(s_{m,k}) &\simeq v^{\theta_t}(s_{m,k}) + \td_t(s_{m,k},s_{m,k+1}) \nonumber\\
&= r(s_{m,k},\pi_{t+1}(s_{m,k}),s_{m,k+1}) + \gamma v^{\theta_t}(s_{m,k+1}).\label{lpivi}
\end{align}
When $\lambda=1$, the generic term becomes the sampled discounted return from $s_{m,k}$ until the end of the trajectory: 
\begin{align}
v^{ \theta_{t+1}}(s_{m,k})  & \simeq  v^{\theta_t}(s_{m,k}) + \sum_{s=k}^{N_{m-1}}\gamma^{s-k} \td_t(s_{m,s},s_{m,s+1}) \nonumber\\
 &=  \sum_{s=k}^{N_{m-1}}\gamma^{s-k} r(s_{m,k},\pi_{t+1}(s_{m,k}),s_{m,k+1}).\label{lpipi}
\end{align}
In other words, for these limit values of $\lambda$, the algorithms
correspond to approximate versions of Value Iteration and Policy
Iteration as described by \citet{ndp}.  Also, as explained by
\citet{ioffe} and already mentioned in the introduction, the
TD($\lambda$) algorithm with linear features described
by \citet[chapter 8.2]{sutton} matches the algorithm we have just described when
the above fitting problem is \emph{approximated} using gradient
iterations after each sample.

We follow the same protocol as originally proposed by
\cite{ioffe}.  Let $w=10$ be the width of the board. We consider
approximating the value function as a linear combination of $2w+1=21$ feature functions:
$$
V^\theta(s)=\theta(0) + \sum_{k=1}^w \theta(k)h_k + \sum_{k=1}^{w-1} \theta(k+w)\Delta h_k + \theta(2w)H + \theta(2w+1)L
$$
where:
\begin{itemize}
\item for all $k \in \{1,2,\cdots,w\}$, $h_k$ is the \emph{height} of the $k^\text{th}$ column of the wall; 
\item for all $k \in \{1,2,\cdots,w-1\}$, $\Delta h_k$ is the \emph{height difference} $|h_k-h_{k+1}|$ between columns $k$ and $k+1$;
\item $H$ is the \emph{maximum wall height} $\max_k h_k$;
\item $L$ is the number of \emph{holes} (the number of empty cells covered by at least one full cell).
\end{itemize}

We started our experiments with the initial following vector: $r(2w)=-10$, $r(2w+1)=-1$ and $r(k)=0$ for all $k<2w$, so that the initial greedy policy scores in the low tens (\cite{ioffe}). We used $M=100$ training games for each policy update.
As $\lambda PI$ is a stochastic algorithm, we ran each experiment 10 times.
\begin{figure}[t]

\begin{minipage}[b]{.49\linewidth}
\begin{center}
  \includegraphics[width=7cm]{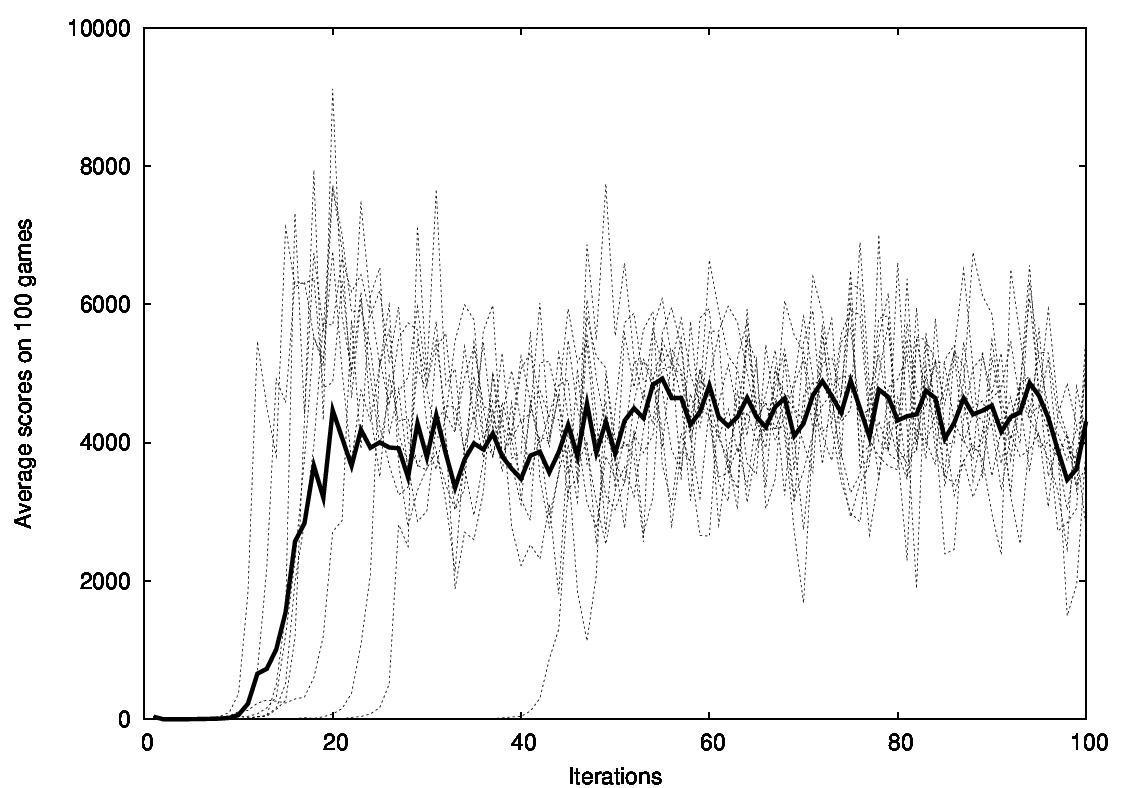}
\end{center}
\end{minipage}	
\begin{minipage}[b]{.49\linewidth}
\begin{center}
  \includegraphics[width=7cm]{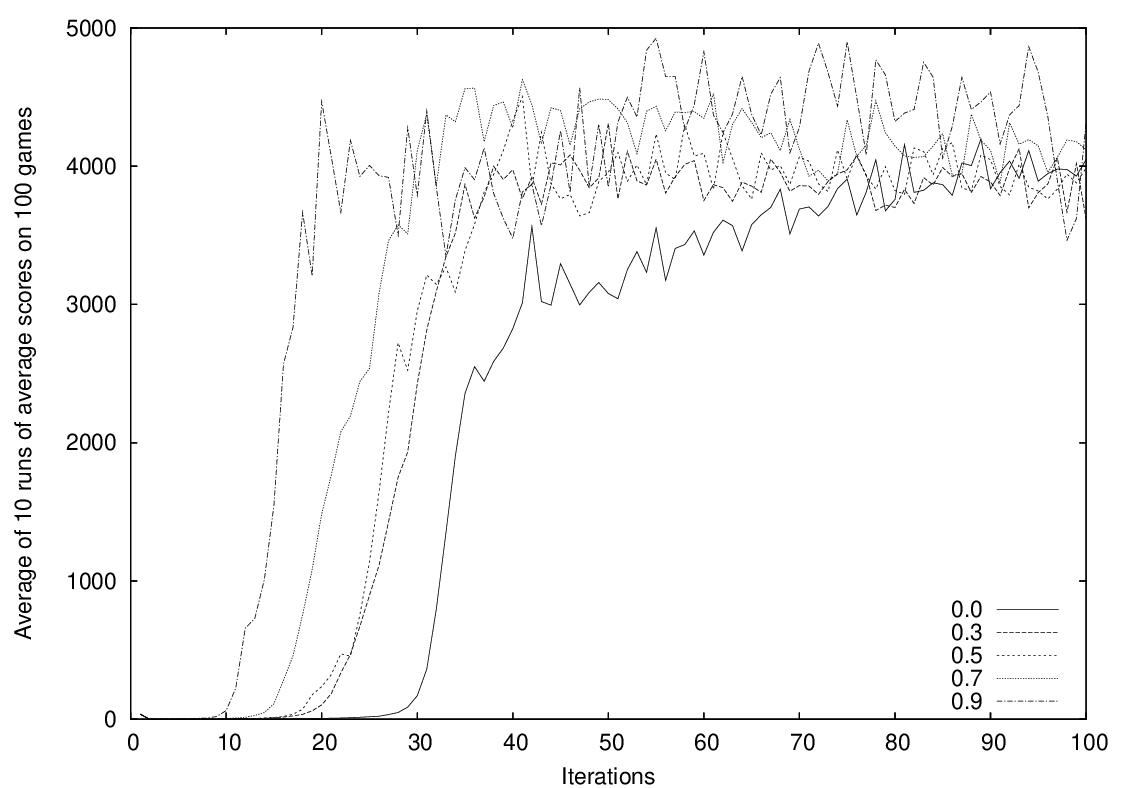}
\end{center}
\end{minipage}	
\caption{Average Score versus the number of iterations {\bf Left}: 10 runs of $\lambda PI$ with $\lambda=0.9$. Each point of each run is the average score computed with $M=100$ games. The dark curve is a pointwise average of the 10 runs. {\bf Right}: Pointwise average of 10 runs of $\lambda PI$ for different values of $\lambda$; the curve which appears to be the best ($\lambda=0.9$) is the same as the bold curve of the left graph. \label{exp1}}
\end{figure}
Figure~\ref{exp1} displays the learning curves. The left
graph shows the 10 runs of $\lambda PI$ (each point is the average
score computed with the $M=100$ games) and the corresponding
pointwise average for a single value of $\lambda$, while the right
graph shows such pointwise average curves for different values of
$\lambda$: 0.0, 0.3, 0.5, 0.7 and 0.9. We chose to display on the
left graph the runs corresponding to the value of $\lambda=0.9$ that
seemed to be the best on the right graph.

We can make the following observations.
\begin{itemize}
\item Though we initialized with not so bad a policy (the first value
  is around 30), the performance first drops to 0 and it \emph{really}
  starts improving after a few iterations (typically arount ten). This
  is due to the fact that the initial value function is really bad:
  with the given parameters, the initial value is everywhere negative
  although it is clear that the optimal value function (the average
  best score) is everywhere positive. Further experiments showed that
  the overall behaviour of the algorithm was not affected by the
  weight initialization.

\item The rise of performance globally happens sooner for larger
  values of $\lambda$, that is for values that makes the algorithm
  closer to Policy Iteration. This is not surprising as it complies
  with the fact that $\lambda$ modulates the speed at which the value
  estimate tracks the real value of the current policy. However, the
  performance did not rise for $\lambda=1$ (when it is equivalent to
  Approximate Policy Iteration), and this is probably due to the fact
  that the variance of the value update is too high.

\item Quantitatively, the scores reach an overall level of 4000 lines
  per games for a big range of values of $\lambda$.

\end{itemize}

The empirical results we have just described qualitatively and
quantitatively differ from the ones that were originally published in
\cite{ioffe}, even though it is the exact same experimental
setup. About their results, the authors wrote: ``\emph{An interesting
  and somewhat paradoxical observation is that a high performance is
  achieved after relatively few policy iterations, but the performance
  gradually drops significantly. We have no explanation for this
  intriguing phenomenon, which occurred with all of the successful
  methods that we tried}''.  As we explain now, we believe that the
``intriguing'' character of the results of \cite{ioffe} might be
related to a subtle implementation difference. Indeed, we can reproduce learning curves that are similar to those of
\cite{ioffe} with a little modification in our implementation of
$\lambda PI$, that removes the special treatments for the terminal
states done through Equations~\ref{pb1} and~\ref{pb2}.  More
precisely, if we replace them by the following Equations:
\begin{equation}
\label{mod1}
v^{ \theta_{t+1}}(s_{m,N_m}) \simeq v^{\theta_t}(s_{m,N_m})
\end{equation}
\vspace{-.5cm}
\begin{equation}
\label{mod2}
\td_t(s_{m,N_m-1},s_{m,N_m})=r(s_{m,N_m-1},\pi_{t+1}(s_{m,N_m-1})) + \gamma v^{\theta_t}(s_{m,N_m})  - v^{\theta_t}(s_{m,N_m-1})
\end{equation}
that is if we replace the terminal value $0$ by the value $V^{\theta_t}(s_{m,N_m})$ which is computed through the features of the terminal wall configuration $s_{m,N_m}$, then we
get the performance shown in Figure~\ref{expbi}.
\begin{figure}[t]
\begin{minipage}[b]{.49\linewidth}
\begin{center}
  \includegraphics[width=7cm]{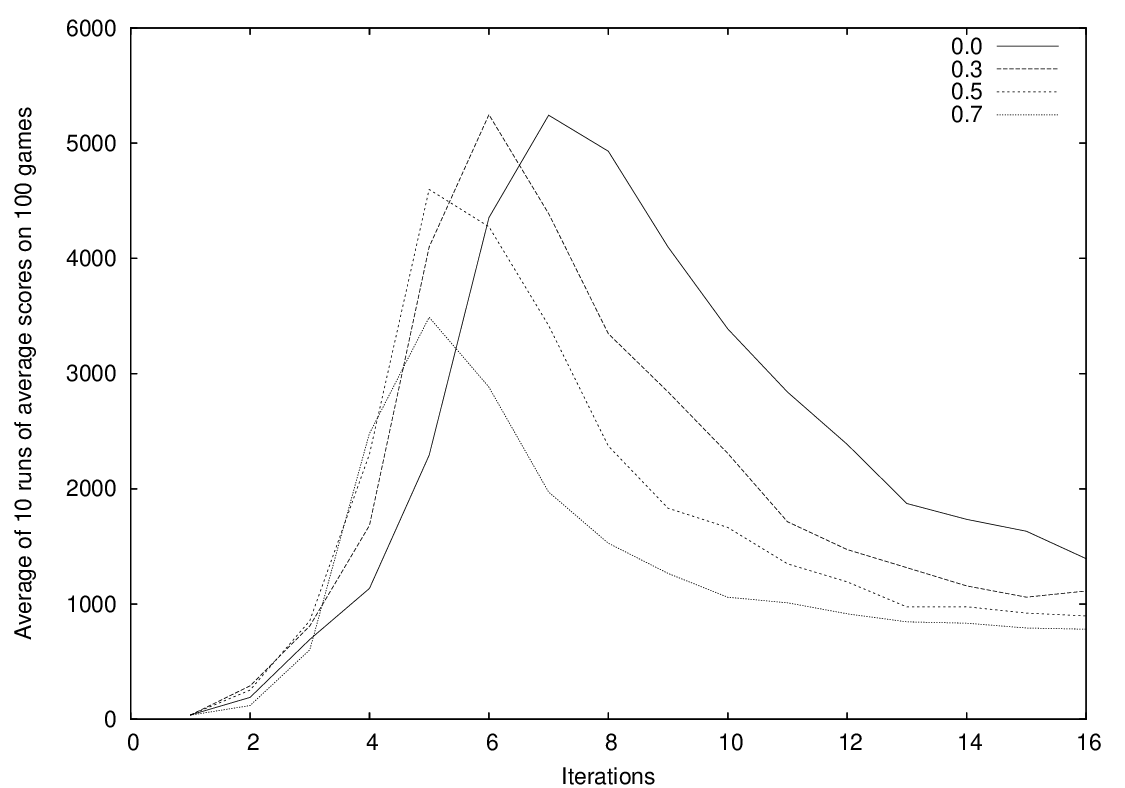}
\end{center}
\end{minipage}	
\begin{minipage}[b]{.49\linewidth}
\begin{center}
  \includegraphics[width=6cm]{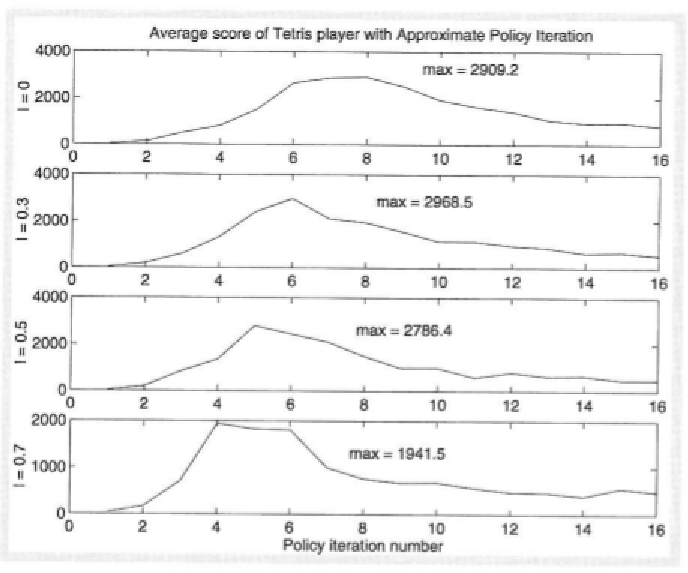}
\end{center}
\end{minipage}
\caption{{\bf Left}: Average score versus the number of iterations of
  $\lambda PI$, \emph{modified} so that it resembles the results of
  \cite{ioffe} (see text for details).  {\bf Right}: A scan of the
  learning curves obtained by \citet{ioffe}.
\label{expbi}}
\end{figure}
This figure shows the performance with respect to the
iterations. 
We observe that the performance
evolution  qualitatively matches the performance
curves published in \cite{ioffe} and illustrates the above quotation
describing the ``intriguing phenomenon''\footnote{The watchful reader may have noticed that the
performance that we obtain is about twice that of \citet{ioffe}. A
close inspection of the Tetris domain description in \citep{ioffe}
shows that the authors consider the game of Tetris on a $10 \times 19$
board instead of our $10 \times 20$ setting, and as argued in a recent review
on Tetris \citep{Thiery:2009}, this small difference is sufficient for
   explaining such a big performance difference.}. 

In such a modified form, the approximate $\lambda PI$ algorithm makes much less sense: in particular, 
it is not true anymore that it reduces to approximate Value Iteration and approximate Policy Iteration when $\lambda=0$ and $\lambda=1$ respectively: Equations~\ref{mod1} and~\ref{mod2} induce a bias so that we cannot recover the identities of Equations~\ref{lpivi} and~\ref{lpipi}. 
A closer examination of these experiments showed that the weights $(\theta_k)$ were diverging.
This is not a surprise, since the use of Equations~\ref{mod1} and~\ref{mod2} violates the condition expressed in Remark~\ref{remarkundiscounted}-1 that there should be no error in the terminal state.

\section{Conclusion and Future Work}

We have considered the \lpi{ }algorithm introduced by \citet{ioffe}
that generalizes the standard algorithms Value Iteration and Policy
Iteration. We have reviewed some results by \citet{puterman},
\citet{ndp} and \cite{munosapi,munosavi}, concerning the rate of
convergence and some performance bounds of these standard algorithms
in the exact and approximate cases. We have extended these results to
\lpi{ }and derived convergence rates and performance bound for this algorithm, some of which
are to our knowledge even new in the cases where \lpi{ }reduces to
Value Iteration and Policy Iteration, notably the bound~\ref{newbound} (page \pageref{newbound}).
 Not only does our analysis
generalize previous results, but it improves them in two ways:
\begin{itemize}
\item  as suggested by the results of \citet{puterman} in the exact case, the
use of the span seminorm has enabled us to derive tighter bounds;
\item our analysis of Approximate \lpi{ }relates the asymptotic
  performance of the algorithm to the \emph{asymptotic}
  errors/residuals instead of a uniform bound of the errors/residuals
  and this might be of practical interest\footnote{Recently and
    independently, \citet{Farahmand:2010} derived bounds that have a
    similar flavour: they highlight the fact that the errors that have
    the more weight on the performance bounds are the latest.}.
\end{itemize}
More generally, we believe that an important contribution of this
paper is of conceptual nature: we provide a unified vision of some of
the main Approximate Dynamic Programming algorithms and their
analyses; in particular, we hope that the new proof technique that is
detailed in the appendices --- especially the different objects that are
defined in our proof overview in Section~\ref{proofoverview} --- shall be
useful for further studies. 

As we mentionned earlier, \citet{munosavi} introduced some
concentration coefficients that are finer than the one we used
throughout the paper.  In the same spirit, \citet{Farahmand:2010}
recently revisited the error propagation of \citet{munosavi,munosapi}
and improved (among other things) the constant in the bound related to
these concentration coefficients.  A natural track would be to adapt
these refinements to the analysis of \lpi. This does not look completely
trivial since the componentwise analysis we derived for \lpi{ }is
significanlty more intricate than the ones we find in the specific
limit cases $\lambda=0$ (Value Iteration) and $\lambda=1$ (Policy
Iteration).

Another potential direction would be to study the implications of the
choice of the parameter $\lambda$, as for instance is done by
\cite{singh} for the value estimation problem. On this matter, the
original analysis by \cite{ioffe} shows how one can concretely
implement Exact \lpi. Each iteration requires the computation of the
fixed point of the $\beta$-contracting operator $M_k$ (see Equation
\ref{mdef2} page \pageref{mdef2}). We plan to study the tradeoff
between the ease for computing this fixed point (the smaller $\beta$
the faster) and the time for \lpi{ }to converge to the optimal policy
(the bigger $\beta$ the faster). In parallel, the reader might have
noticed that most of the bounds we have provided do not depend on
$\lambda$. An interesting question is whether the finer concentration
coefficients of \citet{munosavi} and \citet{Farahmand:2010} we have
just discussed may help keeping track of the influence of $\lambda$ on
the performance of the exact or approximate algorithm. In general,
it would be interesting if we could tie a choice of $\lambda$ to some intrisic characterics of the
 MDP, like for instance its smoothness.

Last but not least, we should insist on the fact that the
implementation that we have described in Section~\ref{thealg}, and
which is borrowed from \cite{ioffe}, is just one possible instance of
\lpi. In the case of linear approximation architectures,
\citet{Thiery:2010} have proposed an implementation of \lpi{ }that is
based on LSPI \citep{Lagoudakis:2003}, in which the fixed point of
$M_k$ is approximated using LSTD(0) \citep{Bradtke:1996}. Recently,
\cite{Bertsekas:2011} proposed to compute this
very fixed point with a variation of LSPE($\lambda'$) \citep{ioffe,Nedic:2003} for some
$\lambda'$ potentially different from $\lambda$.  Because of their
very close structure, any existing implementation of Approximate
Policy Iteration may probably be turned into some implementation of
\lpi. Proposing such implementations and assessing their relative
merits constitutes interesting future research. This may in particular
be done through some finite sample analysis, as had recently been done
for Approximate Value Iteration and Policy Iteration implementations
\citep{Antos:2007a,Antos:2008,Munos:2008,Lazaric:2010}.


\bibliography{biblio}
















\vspace{2cm}

\appendix

\begin{appendix}

\addcontentsline{toc}{section}{Appendix}
\begin{center}
{\huge \bf Appendices}
\end{center}

The following Appendices contains all the proofs concerning the
analysis of \lpi. We write $P_k=P^{\pi_k}$ the stochastic matrix
corresponding to the policy $\pi_k$ which is greedy with respect to
$v_{k-1}$, $P_*$ the stochastic matrix corresponding to the optimal
policy $\pi_*$. Similarly we write $\T_k$ and $\T_*$ the associated
Bellman operators.

The proof techniques we have developped are  inspired by those
of Munos in the articles \citep{munosapi,munosavi}. Most of the
inequalities appear from the definition of the greedy operator:
$$
\pi = \mbox{greedy}(v) \Leftrightarrow \forall \pi', \T^{\pi'}v \leq \T^{\pi}v.
$$
We often use the property that an average of stochastic matrices is also a stochastic matrix. A recurrent instance of this property is: if $P$ is some stochastic matrix, then the geometric average
$$
(1-\alpha)\sum_{i=0}^{\infty}(\alpha P)^i=(1-\alpha)(I-\alpha P)^{-1}
$$
with $0 \leq \alpha < 1$ is also a stochastic matrix.
We use the property that if some vectors $x$ and $y$ are such that $x \leq y$, then $P x \leq P y$ for any stochastic matrix $P$.
Eventually, we will use the following equivalent forms of the operator $\T^{\pi}_\lambda$ (three of them were introduced in page~\pageref{form1}): for any value $v$ and any policy $\pi$, we have
\begin{eqnarray}
\T^{\pi}_\lambda v & := & v + (I-\lambda\gamma P^{\pi})^{-1}(\T^{\pi}v - v) \label{appform1} \\
& = & (I-\lambda\gamma P^{\pi})^{-1}( \T^\pi v - \lambda\gamma P^\pi v ) \label{appform2}  \\
& = & (I-\lambda\gamma P^{\pi})^{-1}(r^{\pi}+ (1-\lambda)\gamma P^{\pi} v ) \label{appform3}  \\
& = & (I-\lambda\gamma P^{\pi})^{-1}(\lambda r^{\pi}+ (1-\lambda)\T^{\pi} v ). \label{appform4}
\end{eqnarray}


\section{Proofs of Lemmas~\ref{lbg}-\ref{lrecd} (core lemmas of the error propagation)}

\label{appcore}

In this section, we prove the series of Lemmas that are at the heart of our analysis of the error propagation
of \lpi.

\subsection{Proof of Lemma~\ref{lbg} (a relation between the \shift{ }and the \bell)}

Using the definition of $w_k=\T_\lambda^{\pi_k}v_{k-1}$ and the formulation of Equation~\ref{appform3}, we can see that we have:
\begin{eqnarray*}
(I-\gamma P_k)\s_k & = & (I-\gamma P_k)(w_k-v^{\pi_k}) \\
& = & (I-\gamma P_k)w_k-r_k \\
& = & (I-\lambda\gamma P_k+\lambda\gamma P_k-\gamma P_k)w_k-r_k \\
& = & (I-\lambda\gamma P_k)w_k + (\lambda\gamma P_k-\gamma P_k)w_k-r_k \\
& = & r_k+(1-\lambda)\gamma P_k v_{k-1} + (\lambda-1)\gamma P_k w_k - r_k \\
& = & (1-\lambda)\gamma P_k(v_{k-1}-w_k)\\
& = & (1-\lambda)\gamma P_k(I-\lambda\gamma P_k)^{-1}(v_{k-1}-\T_k v_{k-1})\\
& = & (1-\lambda)\gamma P_k(I-\lambda\gamma P_k)^{-1}(-\b_{k-1}).
\end{eqnarray*}
Therefore
$$
\s_k =  \beta (I-\gamma P_k)^{-1} A_k (-\b_{k-1})
$$
with 
$$
A_k:=(1-\lambda\gamma)P_{k}(I-\lambda\gamma P_k)^{-1}. 
$$

Suppose that we have a lower bound of the \bell: $\b_{k-1} \geq \underline{\b_{k-1}}$ (we shall derive one soon). Since $(I-\gamma P_k)^{-1}A_k$ only has non-negative elements then
$$
\s_k \leq \beta(I-\gamma P_k)^{-1}A_k (-\underline{\b_{k-1}}):=\overline{\s_k}.
$$

\subsection{Proof of Lemma~\ref{lrecg} (a lower bound of the \bell)}

From the definition of the algorithm, and using the fact that $\T_k v^{\pi_k}=v^{\pi_k}$, we see that:
\begin{eqnarray}
\b_{k} & = & \T_{k+1} v_k - v_k  \nonumber\\
& = & \T_{k+1} v_k - \T_k v_k + \T_k v_k - v_k \nonumber\\
& \geq & \T_k v_k - v_k \nonumber\\
& = & \T_k v_k - \T_k v^{\pi_k} + v^{\pi_k} - v_k \nonumber\\
& = & \gamma P_k(v_k - v^{\pi_k}) + v^{\pi_k} - v_k \nonumber\\
& = & (\gamma P_k-I)(\s_k+\epsilon_k).\nonumber \\
& = & \beta A_k \b_{k-1} + (\gamma P_k-I)\epsilon_k
\end{eqnarray}
where we eventually used the relation between $\s_k$ and $\b_k$ (Lemma~\ref{lbg}).
In other words:
$$
\b_{k+1} \geq \beta A_{k+1}\b_k+x_{k+1}
$$
with
$$
x_k:=(\gamma P_k-I)\epsilon_k. 
$$
Since $A_k$ is a stochastic matrix and $\beta \geq 0$, we get by induction:
$$
\b_k \geq \sum_{j=k_0+1}^{k} \beta^{k-j}\left(A_{k}A_{k-1}...A_{j+1}\right) x_{j} + \beta^{k-k_0}\left(A_{k} A_{k-1}...A_{k_0+1}\right)\b_{k_0} := \underline{\b_k}. 
$$

\subsection{Proof of Lemma~\ref{lrecd} (an upper bound of the distance)}

Given that $\T_*v_*=v_*$, we have
\begin{eqnarray*}
v_*& = & v_*+(I-\lambda\gamma P_{k+1})^{-1}(\T_*v_*-v_*) \\
   & = & (I-\lambda\gamma P_{k+1})^{-1}(\T_*v_*-\lambda\gamma P_{k+1}v_*).
\end{eqnarray*}
Using the definition of $w_{k+1}=\T_\lambda^{\pi_{k+1}}v_{k}$ and the formulation of Equation~\ref{appform2}, one can see that the distance satisfies:
\begin{eqnarray*}
d_{k+1} & = & v_*-w_{k+1} \\
& = & (I-\lambda\gamma P_{k+1})^{-1}[(\T_*v_*-\lambda\gamma P_{k+1}v_*)-(\T_{k+1}v_k-\lambda\gamma P_{k+1}v_k)] \\
& = & (I-\lambda\gamma P_{k+1})^{-1}[\T_*v_*-\T_{k+1}v_k+\lambda\gamma P_{k+1}(v_k-v_*)]\\
& = & \lambda\gamma P_{k+1}d_{k+1} + \T_*v_*-\T_{k+1}v_k+\lambda\gamma P_{k+1}(v_k-v_*) \\
& = & \lambda\gamma P_{k+1}d_{k+1} + \T_*v_*-\T_{k+1}v_k+\lambda\gamma P_{k+1}(w_k+\epsilon_k-v_*) \\
& = & \lambda\gamma P_{k+1}d_{k+1} +  \T_* v_*-\T_{k+1} v_k + \lambda\gamma P_{k+1} (\epsilon_k-d_k) \\
& = &  \T_* v_*-\T_{k+1} v_k + \lambda\gamma P_{k+1}(\epsilon_k+d_{k+1}-d_k).
\end{eqnarray*}
Since $\pi^{k+1}$ is greedy with respect to $v_k$, we have $\T_{k+1} v_k \geq \T_*v_k$ and therefore:
\begin{eqnarray*}
\T_* v_*-\T_{k+1} v_k & = &  \T_*  v_* -\T_*v_k + \T_*v_k - \T_{k+1} v_k \\
& \leq &  \T_* v_* -\T_*v_k \\
& = & \gamma P_* (v_*-v_k) \\
& = & \gamma P_* (v_*-(w_k+\epsilon_k)) \\
& = & \gamma P_* d_k-\gamma P_*\epsilon_k.
\end{eqnarray*}
As a consequence, the distance satisfies:
$$
d_{k+1} \leq  \gamma P_* d_k + \lambda\gamma P_{k+1}(\epsilon_{k} + d_{k+1}-d_k)-\gamma P_*\epsilon_k.
$$
Noticing that:
\begin{eqnarray*}
\epsilon_{k}+d_{k+1}-d_k & = &\epsilon_k+w_k-w_{k+1}\\
& =& v_k-w_{k+1}\\
& = & -(I-\lambda\gamma P_{k+1})^{-1}(\T_{k+1}v_k-v_k) \\
& =& (I-\lambda\gamma P_{k+1})^{-1}(-\b_{k}) \\
& \leq& (I-\lambda\gamma P_{k+1})^{-1}(- \underline{\b_{k}}),
\end{eqnarray*}
we get:
$$
d_{k+1} \leq \gamma P_* d_k + y_k
$$
where 
$$
y_k:=\frac{\lambda\gamma}{1-\lambda\gamma}A_{k+1}(-\underline{\b_{k}})-\gamma P_*\epsilon_k. 
$$
Since $P_*$ is a stochastic matrix and $\gamma \geq 0$, we have by induction:
$$
d_k \leq \sum_{j=k_0}^{k-1} \gamma^{k-1-j}(P_*)^{k-1-j} y_{j} + \gamma^{k-k_0}(P_*)^{k-k_0}d_{k_0} = \overline{d_k}.
$$

%



\section{Proofs of Lemma~\ref{thexact} (performance of Exact \lpi)}

\label{appelpi}

We here derive the convergence rate bounds for Exact \lpi{ }(as expressed in Lemma~\ref{thexact} page \pageref{thexact}).
We rely on the loss bound analysis of Appendix~\ref{appcore} with $\epsilon_k=0$.
In this specific case, we know that the loss $l_k \leq \overline{d_k}+\overline{\s_k}$ where
\begin{eqnarray*}
-\underline{\b_k} & = & \beta^{k-k_0}A_{k}A_{k-1}...A_{k_0+1}(-\b_{k_0}), \\
\overline{d_k} & = & \frac{\lambda \gamma}{1-\lambda\gamma} \sum_{j=k_0}^{k-1}\gamma^{k-1-j}(P_*)^{k-1-j}A_{j+1}(-\underline{\b_{j}}) + \gamma^{k-k_0}(P_*)^{k-k_0}d_{k_0}, \\
\mbox{and~~}\overline{\s_k} & = & \beta(I-\gamma P_k)^{-1}A_k (-\underline{\b_{k-1}}).
\end{eqnarray*}
Introducing the following stochastic matrices:
\begin{align}
X_{i,j,k}&:=(P_*)^{k-1-i}A_{i+1}A_{i}...A_{j+1} \nonumber\\
\mbox{and~~}Y_{j,k}&:= (1-\gamma)(I-\gamma P_k)^{-1}A_k A_{k-1}...A_{j+1}, \nonumber
\end{align}
we have
$$
\overline{d_k} = \frac{\lambda \gamma}{1-\lambda\gamma} \sum_{j=k_0}^{k-1}\gamma^{k-1-j}\beta^{j-k_0}X_{j,k_0,k}(-\b_{k_0}) + \gamma^{k-k_0}(P_*)^{k-k_0}d_{k_0} 
$$
and 
$$
\overline{\s_k} = \frac{\beta^{k-k_0}}{1-\gamma}Y_{k_0,k}(-\b_{k_0}).
$$
Therefore the loss satisfies:
\begin{align}
l_k &\le \overline{d_k}+\overline{\s_k} \nonumber\\
& \le \left(\frac{\gamma^{k-k_0}}{1-\gamma}\right)E'_{kk_0}(-\b_{k_0}) + \gamma^{k-k_0}(P_*)^{k-k_0}d_{k_0}\label{l_ref}
\end{align}
with
$$
E'_{kk_0}:=\left(\frac{1-\gamma}{\gamma^{k-k_0}}\right)\left(\frac{\lambda \gamma }{1-\lambda\gamma} \sum_{j=k_0}^{k-1}\gamma^{k-1-j}\beta^{j-k_0} X_{j,{k_0},k} + \frac{\beta^{k-k_0}}{1-\gamma}Y_{{k_0},k}\right).
$$
To end the proof, we simply need to prove the following lemma:
{\begin{lemma}
$E'_{kk_0}$ is a stochastic matrix.
\end{lemma}}
\begin{proof}
It is clear from the definition of $X_{i,j,k}$ and $Y_{j,k}$ that normalizing $E'_{kk_0} $ gives stochastic matrices. So we just need to check that their max norm is 1.
\begin{eqnarray*}
\|E'_{kk_0}\|_\infty &= & \frac{1-\gamma}{\gamma^{k-k_0}} \left( \frac{\lambda \gamma }{1-\lambda\gamma} \sum_{j=k_0}^{k-1}\gamma^{k-1-j}\beta^{j-k_0} + \frac{\beta^{k-k_0}}{1-\gamma} \right) \\
& = & \frac{1-\gamma}{\gamma^{k-k_0}} \left(\frac{\lambda \gamma }{1-\lambda\gamma}\frac{\gamma^{k-k_0}-\beta^{k-k_0}}{\gamma-\beta} + \frac{\beta^{k-k_0}}{1-\gamma} \right)\\
& = & \frac{1-\gamma}{\gamma^{k-k_0}} \left( \frac{\gamma^{k-k_0}-\beta^{k-k_0}}{1-\gamma}+\frac{\beta^{k-k_0}}{1-\gamma} \right)\\
& = & 1
\end{eqnarray*}
where we used the facts that $\frac{\lambda\gamma}{\gamma-\beta}=\frac{1}{1-\beta}$ and $(1-\beta)(1-\lambda\gamma)=1-\gamma$.  
\end{proof}

\subsection{Proof of Equation~\ref{thexacteq2} (a bound with respect to the \bell)}

We first need the following lemma:
{\begin{lemma}
\label{dg}
The bias and the distance are related as follows:
$$
\b_{k} \geq (I-\gamma P_*)d_k.
$$
\end{lemma}}
\begin{proof}
Since $\pi_{k+1}$ is greedy with respect to $v_k$, $\T_{k+1}v_k \geq \T_* v_k$ and
\begin{eqnarray*}
\b_{k} & = & \T_{k+1}v_k - v_k  \\
& = & \T_{k+1}v_k - \T_* v_k + \T_* v_k - \T_*v_* + v_* - v_k  \\
& \geq & \gamma P_*(v_k - v_*) + v_* - v_k  \\
& = & (I-\gamma P_*)d_k. \qedhere
\end{eqnarray*}
~\vspace{-1.3cm}

\end{proof}

We thus have: 
$$
d_{k_0} \leq (I-\gamma P_*)^{-1}\b_{k_0}.
$$
Then Equation~\ref{l_ref} becomes
\begin{eqnarray*}
l_k & \leq & \left[ \gamma^{k-k_0}(P_*)^{k-k_0}(I-\gamma P_*)^{-1} -\left(\frac{\gamma^{k-k_0}}{1-\gamma}\right)E'_{kk_0}\right] \b_{k_0} \\
& = & \frac{\gamma^{k-k_0}}{1-\gamma}\left[ E_{kk_0}-E'_{kk_0} \right] \b_{k_0} 
\end{eqnarray*}
where:
$$
E_{kk_0}  :=  (1-\gamma)(P_*)^{k-k_0}(I-\gamma P_*)^{-1}
$$
is a stochastic matrix. 

\subsection{Proof of Equation~\ref{thexacteq1} (a bound with respect to the distance)}

From Lemma~\ref{dg}, we know that 
$$
-\b_{k_0} \leq (I-\gamma P_*)(-d_{k_0}).
$$
Then, Equation~\ref{l_ref} becomes
\begin{eqnarray*}
l_k &\leq & \left[\gamma^{k-k_0}(P_*)^{k-k_0}-\left(\frac{\gamma^{k-k_0}}{1-\gamma}\right)E'_{kk_0}(I-\gamma P_*) \right] d_{k_0} \\
& = &\frac{\gamma^{k-k_0}}{1-\gamma}\left[ F_{kk_0}-E'_{kk_0} \right] d_{k_0}
\end{eqnarray*}
where
$$
F_{kk_0} := (1-\gamma) P_*^{k-k_0}+\gamma E'_{kk_0} P_*
$$
is a stochastic matrix. 

\subsection{Proof of Equation~\ref{thexacteq3} (a bound with respect to the distance and the loss of the greedy policy)}

Define $\hat{v}_{k_0}:=v_{k_0}-K\e$ where $K$ is some constant.
The following statements are equivalent:
\begin{eqnarray*}
\hat{b}_{k_0} & \geq & 0 \\
\T_{k_0+1} \hat{v}_{k_0} & \geq & \hat{v}_{k_0} \\
r_{k_0+1}+ \gamma P_{k_0+1}(v_{k_0}-K\e) & \geq & v_{k_0}-K\e \\
(I-\gamma P_{k_0+1})K\e & \geq & -r_{k_0+1}+(I-\gamma P_{k_0+1})v_{k_0} \\
K\e & \geq & (I-\gamma P_{k_0+1})^{-1}(-r_{k_0+1})+v_{k_0} \\
K\e & \geq & v_{k_0}-v^{\pi_{k_0+1}}.
\end{eqnarray*}
The minimal $K$ for which $\hat{b}_{k_0} \geq 0$  is thus $K:=\max_s[v_{k_0}(s)-v^{\pi_{k_0+1}}(s)]$.
As $\hat{v}_{k_0}$ and $v_{k_0}$ only differ by a constant vector, they generate the same sequence of policies $\pi_{k_0+1}, \pi_{k_0+2}...$
Then, as $\hat{b}_{k_0} \geq 0$, Equation~\ref{l_ref} tells us that
\begin{eqnarray*}
v_*-v^{\pi_k} & \leq & \gamma^{k-k_0}(P_*)^{k-k_0}(v_*-\hat{v}_{k_0})\\
& = & \gamma^{k-k_0}(P_*)^{k-k_0}(v_*-v_{k_0}+K\e).
\end{eqnarray*}
Now notice that
\begin{eqnarray*}
K & = & \max_s[v_{k_0}(s)-v_*(s)+v_*(s)-v^{\pi_{k_0+1}}(s)] \\
& \leq & \max_s[v_{k_0}(s)-v_*(s)]+\max_s[v_*(s)-v^{\pi_{k_0+1}}(s)] \\
& = & -\min_s[v_*(s)-v_{k_0}(s)]+\norm{v_*(s)-v^{\pi_{k_0+1}}}{\infty}.
\end{eqnarray*}
Then, using the fact that $(P_*)^{k-k_0}\e=\e$, we get:
$$
v_*-v^{\pi_k} \leq \gamma^{k-k_0}\left[(P_*)^{k-k_0}\left((v_*-v_{k_0})-\min_s[v_*(s)-v_{k_0}(s)]e\right) + \norm{v_*(s)-v^{\pi_{k_0+1}}}{\infty} \e\right]. 
$$


\section{Proofs of Equation~\ref{ccbbp0} in Lemma~\ref{th} (componentwise bounds on the error propagation)}

\label{appalpi}
We here use the loss bound analysis of Appendix~\ref{appcore} to derive an asymptotic analysis of approximate \lpi{ }with respect to the approximation error. The results stated here constitute a proof of the first inequality of Lemma~\ref{th} page \pageref{th}.

\subsection{Proof of Equation~\ref{ccbbp0} }


Since the loss satisfies
\begin{equation}
\label{lds}
l_k=d_k+\s_k \leq \overline{d_k}+\overline{\s_k},
\end{equation}
an upper bound of the loss can be derived from the upper bound of the distance and the \shift.

Let us first concentrate on the bound $\overline{d_k}$ of the distance.
Lemmas \ref{lrecg} and \ref{lrecd} imply that:
\begin{eqnarray}
\overline{d_k} & = &\sum_{i=k_0}^{k-1} \gamma^{k-1-i}(P_*)^{k-1-i} y_{i} +  \O(\gamma^{k-k_0}), \nonumber \\
y_i &= & \frac{\lambda\gamma}{1-\lambda\gamma}A_{i+1}(-\underline{\b_{i}})-\gamma P_*\epsilon_i, \nonumber\\
-\underline{\b_{i}}& =& \sum_{j=k_0}^{i} \beta^{i-j}\left(A_{i} A_{i-1}...A_{j+1}\right) (-x_{j}) + \O(\beta^{i-k_0}), \label{niac} \\
\mbox{and~~}-x_j& =&(I-\gamma P_j )\epsilon_j. \nonumber
\end{eqnarray}
Writing
$$
X_{i,j,k}:=(P_*)^{k-1-i}A_{i+1}A_{i}...A_{j+1}
$$
and putting all things together, we see that:
\begin{eqnarray}
\overline{d_k} & = & \frac{\lambda\gamma}{1-\lambda\gamma} \sum_{i=k_0}^{k-1} \gamma^{k-1-i} \left( \sum_{j=k_0}^{i} \beta^{i-j}X_{i,j,k}(I-\gamma P_j) \epsilon_j + \O(\beta^{i-k_0}) \right) \nonumber\\
& & \hspace{3cm}- \sum_{i=k_0}^{k-1}\gamma^{k-i}(P_*)^{k-i}\epsilon_i + \O(\gamma^{k-k_0}) \nonumber \\
& = & \frac{\lambda\gamma}{1-\lambda\gamma} \sum_{i=k_0}^{k-1} \sum_{j=k_0}^{i} \gamma^{k-1-i}\beta^{i-j}X_{i,j,k}(I-\gamma P_j) \epsilon_j - \sum_{i=k_0}^{k-1}\gamma^{k-i}(P_*)^{k-i}\epsilon_i + \O(\gamma^{k-k_0})\nonumber \\
& = & \frac{\lambda\gamma}{1-\lambda\gamma} \sum_{j=k_0}^{k-1}\sum_{i=j}^{k-1} \gamma^{k-1-i}\beta^{i-j}X_{i,j,k}(I-\gamma P_j) \epsilon_j - \sum_{j=k_0}^{k-1}\gamma^{k-j}(P_*)^{k-j}\epsilon_j + \O(\gamma^{k-k_0})\nonumber \\
& = &  \sum_{j=k_0}^{k-1} \left[\left( \frac{\lambda\gamma}{1-\lambda\gamma} \sum_{i=j}^{k-1}\gamma^{k-1-i}\beta^{i-j}X_{i,j,k}(I-\gamma P_j) \right)- \gamma^{k-j}(P_*)^{k-j} \right] \epsilon_j  + \O(\gamma^{k-k_0}) ~~~~~\label{dbound}
\end{eqnarray}
where between the first two lines, we used the fact that:
\begin{equation}
\label{sumgb}
\frac{\lambda\gamma}{1-\lambda\gamma} \sum_{i=k_0}^{k-1}\gamma^{k-1-i}\beta^{i-k_0} = \frac{\lambda\gamma}{1-\lambda\gamma} \frac{\gamma^{k-k_0}-\beta^{k-k_0}}{\gamma-\beta}=\frac{\gamma^{k-k_0}-\beta^{k-k_0}}{1-\gamma}=\O(\gamma^{k-k_0})
\end{equation}
using the identities ${\lambda\gamma}=\frac{\gamma-\beta}{1-\beta}$ and $1-\gamma\lambda=\frac{1-\gamma}{1-\beta}$.

Let us now consider the bound $\overline{\s_k}$ of the \shift. From Lemma~\ref{lbg} and the bound on $b_k$ in Equation~\ref{niac}, we have
\begin{eqnarray}
\overline{\s_k} & = & \beta(I-\gamma P_k)^{-1}A_k (-\underline{\b_{k-1}} ) \nonumber \\
& = & \beta(I-\gamma P_k)^{-1}A_k \left[ \left( \sum_{j=k_0}^{k-1} \beta^{k-1-j}\left(A_{k-1}A_{k-2}...A_{j+1}\right) (-x_{j}) \right) + \O(\gamma^{k-k_0}) \right] \nonumber \\
& = & \sum_{j=k_0}^{k-1} \frac{\beta^{k-j}}{1-\gamma} Y_{j,k}(I-\gamma P_j) \epsilon_j  + \O(\gamma^{k-k_0}) \label{bbound}
\end{eqnarray}
with
$$
Y_{j,k}:= (1-\gamma)(I-\gamma P_k)^{-1}A_k A_{k-1}...A_{j+1}.
$$
Eventually, from Equations~\ref{lds},~\ref{dbound} and~\ref{bbound} we get:
{
\begin{align}
l_k &\leq  \sum_{j=k_0}^{k-1} \left[ \left( \frac{\lambda\gamma}{1-\lambda\gamma} \sum_{i=j}^{k-1}\gamma^{k-1-i}\beta^{i-j}X_{i,j,k}+ \frac{\beta^{k-j}}{1-\gamma} Y_{j,k}\right)(I-\gamma P_j) - \gamma^{k-j}(P_*)^{k-j} \right] \epsilon_j  \nonumber\\
& \hspace{10cm}+\O(\gamma^{k-k_0}) \label{lbound}.
\end{align}
}
Introduce the following matrices:
\begin{eqnarray*}
B_{jk} & := & \frac{1-\gamma}{\gamma^{k-j}}\left[ \frac{\lambda\gamma}{1-\lambda\gamma} \sum_{i=j}^{k-1}\gamma^{k-1-i}\beta^{i-j}X_{i,j,k}+ \frac{\beta^{k-j}}{1-\gamma} Y_{j,k}\right] \\
B'_{jk}& := & \gamma B_{jk} P_j + (1-\gamma)(P_*)^{k-j}.
\end{eqnarray*}
{\begin{lemma}
\label{fnorms}
$B_{jk}$ and $B'_{jk}$ are stochastic matrices.
\end{lemma}}
\begin{proof}
It is clear from the definition of $X_{i,j,k}$ and $Y_{j,k}$ that normalizing $B_{jk}$ and $B'_{jk}$ gives stochastic matrices.
So we just need to check that their max norm is 1.
\begin{eqnarray*}
\|B_{jk}\| & = & \frac{(1-\gamma)}{\gamma^{k-j}}\left[ \frac{\lambda\gamma}{1-\lambda\gamma} \sum_{i=j}^{k-1}\gamma^{k-1-i}\beta^{i-j}+ \frac{\beta^{k-j}}{1-\gamma} \right] \\
& = & \frac{(1-\gamma)}{\gamma^{k-j}}\left[ \frac{\lambda\gamma}{1-\lambda\gamma} \frac{\gamma^{k-j}-\beta^{k-j}}{\gamma-\beta}+ \frac{\beta^{k-j}}{1-\gamma} \right] \\
& = & \frac{(1-\gamma)}{\gamma^{k-j}}\left[ \frac{\gamma^{k-j}-\beta^{k-j}}{(1-\lambda\gamma)(1-\beta)}+ \frac{\beta^{k-j}}{1-\gamma} \right] \\
& = & \frac{(1-\gamma)}{\gamma^{k-j}}\left[ \frac{\gamma^{k-j}-\beta^{k-j}}{1-\gamma}+ \frac{\beta^{k-j}}{1-\gamma} \right] \\
& = & 1.
\end{eqnarray*}
where we used the identities: ${\lambda\gamma}=\frac{\gamma-\beta}{1-\beta}$ and $(1-\beta)(1-\gamma\lambda)=1-\gamma$. Then it is also clear that $\|B'_{jk}\|=1$. 
\end{proof}

Thus, Equation~\ref{lbound} can be rewritten as follows:
\begin{eqnarray*}
l_k & \leq & \sum_{j=k_0}^{k-1} \left[\frac{\gamma^{k-j}}{1-\gamma}B_{jk}(I-\gamma P_j)-\gamma^{k-j} (P_*)^{k-j}\right]\epsilon_j  + \O(\gamma^{k-k_0}) \\ 
& = & \frac{1}{1-\gamma}\sum_{j=k_0}^{k-1} \gamma^{k-j}\left[B_{jk}-B'_{jk}\right]\epsilon_j  + \O(\gamma^{k-k_0}).
\end{eqnarray*}
Taking the supremum limit, we see that for all $k_0$,
\begin{equation}
\label{lbound2}
\ls{k} l_k \leq \frac{1}{1-\gamma}\ls{k}  \sum_{j=k_0}^{k-1} \gamma^{k-j}\left[ B_{jk} - B'_{jk}\right]\epsilon_j.
\end{equation}

%


\section{Proofs of Equations~\ref{ccbbp1}-\ref{ccbbp2} in Lemma~\ref{th} (componentwise bounds with respect to the Bellman residuals)}

\label{appbell}

In this section, we study the loss
$$
l_k := v_* - v^{\pi_k}
$$
with respect to the two following {\bf Bellman residuals}:
$$
\pb_k:=\T_{k} v_k-v_k
$$
$$
\mbox{and~~}\b_k:=\T_{k+1} v_k-v_k=\T v_k-v_k.
$$
The term $\pb_k$ says how much $v_k$ differs from the value of $\pi_k$ while $\b_k$ says how much $v_k$ differs from the value of the policies $\pi_{k+1}$ and $\pi_*$. The results stated here prove the last two inequalities of Lemma~\ref{th} page \pageref{th}.

\subsection{Proof of Equation~\ref{ccbbp1} (bounds with respect to the \pbell)}

Our analysis relies on the following lemma
{\begin{lemma}
Suppose that we have a policy $\pi$, a function $v$ that is an approximation of the value $v^\pi$ of $\pi$ in the sense that its residual $\pb:=T^{\pi}v-v$ is small. Taking the greedy policy $\pi'$ with respect to $v$  reduces the loss as follows:
$$
v_*-v^{\pi'} \leq \gamma P_*(v_*-v^{\pi})+\left(\gamma P_*(I-\gamma P)^{-1}-\gamma P'(I-\gamma P')^{-1}\right)\pb
$$
where $P$ and $P'$ are the stochastic matrices which correspond to $\pi$ and $\pi'$.
\end{lemma}}

\begin{proof}
We have:
\begin{eqnarray}
v_* - v^{\pi'} & = & \T_* v_* - \T^{\pi'} v^{\pi'} \nonumber \\
& = & \T_* v_* - \T_* v^\pi + \T_* v^\pi - \T_*v + \T_*v - \T^{\pi'}v + \T^{\pi'}v - \T^{\pi'} v^{\pi'} \nonumber \\
& \leq & \gamma P_* (v_* -v^\pi) + \gamma P_*(v^\pi-v)+\gamma P'(v-v^{\pi'}) \label{br1}
\end{eqnarray}
where we used the fact that $\T_*v \leq \T^{\pi'}v$.
One can see that:
\begin{eqnarray}
v^{\pi}-v & = & \T^{\pi} v^\pi-v  \nonumber \\
& = & \T^{\pi} v^\pi - \T^{\pi} v +  \T^{\pi} v -v  \nonumber \\
& = & \gamma P(v^{\pi}-v) + \pb \nonumber \\
& = & (I-\gamma P)^{-1}\pb\label{br2}
\end{eqnarray}
and that
\begin{eqnarray}
v-v^{\pi'} & = & v - \T^{\pi'} v^{\pi'} \nonumber \\
& = & v - \T^{\pi} v + \T^{\pi} v - \T^{\pi'} v + \T^{\pi'} v - \T^{\pi'}v^{\pi'} \nonumber\\
& \leq & -\pb + \gamma P'(v-v^{\pi'}) \nonumber\\
& \leq & (I-\gamma P')^{-1}(-\pb) \label{br3}.
\end{eqnarray}
where we used the fact that $\T^{\pi} v \leq  \T^{\pi'} v$.
We get the result by putting back Equations~\ref{br2} and~\ref{br3} into Equation~\ref{br1}. \end{proof}

To derive a bound for \lpi, we simply apply the above lemma to $\pi=\pi_k$, $v=v_k$ and $\pi'=\pi_{k+1}$. We thus get:
$$
l_{k+1} \leq \gamma P_* l_k + \left(  \gamma P_*(I-\gamma P_k)^{-1} -\gamma P_{k+1} (I-\gamma P_{k+1})^{-1} \right)\pb_k.
$$
Introduce the following stochastic matrices:
$$
C_k:=(1-\gamma)^2 (I-\gamma P_*)^{-1}\left( P_*(I-\gamma P_k)^{-1} \right),
$$
$$
C'_k:=(1-\gamma)^2 (I-\gamma P_*)^{-1}\left(  P_{k+1}(I-\gamma P_{k+1})^{-1} \right).
$$
This leads to the following componentwise bound:
$$
\ls{k} l_k \leq \frac{\gamma}{(1-\gamma)^2}\ls{k} \left[C_k - C'_k\right] \pb_k .
$$



\subsection{Proof of Equation~\ref{ccbbp2} (bounds with respect to the \bell)}

We rely on the following lemma (which is for instance proved by \citet{munosavi})
{\begin{lemma}
Suppose that we have a function $v$. Let $\pi$ be the greedy policy with respect to $v$. Then
$$
v_*-v^{\pi} \leq \gamma\left[P_*(I-\gamma P_*)^{-1}-P^\pi(I-\gamma P^\pi)^{-1}\right](\T^\pi v-v).
$$
\end{lemma}}
We provide a proof for the sake of completeness:
\begin{proof}
Using the fact that $\T_* v \leq \T^\pi v$, we see that 
\begin{eqnarray*}
v_* - v^{\pi} & = & \T_* v_* - \T^\pi v^{\pi} \\
& = & \T_* v_* -\T_* v + \T_* v - \T^\pi v + \T^\pi v - \T^\pi v^{\pi} \\
& \leq & \T_* v_* -\T_* v + \T^\pi v - \T^\pi v^{\pi}\\
& = & \gamma P_*(v_*-v) + \gamma P^\pi (v-v^\pi)\\
& = & \gamma P_*(v_*-v^\pi) + \gamma P_*(v^\pi-v) \gamma P^\pi (v-v^\pi)\\
& \leq & (I-\gamma P_*)^{-1}(\gamma P_*-\gamma P^\pi)(v^\pi-v).
\end{eqnarray*}
Using Equation~\ref{br2} we see that:
$$
v^\pi-v=(I-\gamma P^\pi)^{-1}(\T^\pi v-v).
$$
Thus
\begin{eqnarray*}
v_* - v^{\pi} & \leq & (I-\gamma P_*)^{-1}(\gamma P_*-\gamma P^\pi)(I-\gamma P^\pi)^{-1}(\T^\pi v-v) \\
& = & (I-\gamma P_*)^{-1}(\gamma P_*-I+I-\gamma P^\pi)(I-\gamma P^\pi)^{-1}(\T^\pi v-v) \\
& = & \left[(I-\gamma P_*)^{-1}-(I-\gamma P^\pi)^{-1}\right](\T^\pi v-v) \\
& = & \gamma\left[P_*(I-\gamma P_*)^{-1}-P^\pi(I-\gamma P^\pi)^{-1}\right](\T^\pi v-v). \qedhere 
\end{eqnarray*}
~\vspace{-1.3cm}

\end{proof}

To derive a bound for \lpi, we simply apply the above lemma to $v=v_{k-1}$ and $\pi=\pi_{k}$. We thus get:
\begin{equation}
\label{bb}
l_{k} \leq \frac{\gamma}{1-\gamma}\left[ D - D'_k \right]\b_{k-1}
\end{equation}
where
\begin{eqnarray*}
D & := & (1-\gamma)P_*(I-\gamma P_*)^{-1} \\
\mbox{and~~}D'_k & := & (1-\gamma)P_{k}(I-\gamma P_{k})^{-1}
\end{eqnarray*}
are stochastic matrices.


\section{Proofs of Corollary~\ref{valueorpolicyconvergence}}
\label{valuepolicyconv}

Ths section provides a proof of Corollary~\ref{valueorpolicyconvergence} page \pageref{valueorpolicyconvergence}, 
in which we refine the bounds when the value or the policy converges.

\subsection{Proof of the first inequality of Corollary~\ref{valueorpolicyconvergence} (when the value converges)}
\label{valueconverge}

Suppose that \lpi{ }converges to some value ${v}$. Let policy ${\pi}$ be the corresponding greedy policy, with stochastic matrix $ P$. Let $\b$ be the \bell{ }of ${v}$. It is also clear that the approximation error also converges to some ${\epsilon}$. Indeed from Algorithm~\ref{algo:lpi} and Equation~\ref{form1}, we get:
$$
{b}=\T  v -  v = (I-\lambda\gamma {P})(-{\epsilon}).
$$
From the bound with respect to the \bell{ }(Equation~\ref{bb} page \pageref{bb}), we can see that:
\begin{eqnarray*}
v_*-v^{{\pi}} & \leq & \left[ (I-\gamma P_*)^{-1} -(I-\gamma {P})^{-1}\right]{b} \\
& = & \left[ (I-\gamma {P})^{-1}-(I-\gamma P_*)^{-1} \right](I-\lambda\gamma {P}){\epsilon} \\
& = & \left[ (I-\gamma {P})^{-1}(I-\lambda\gamma {P})-(I-\gamma P_*)^{-1}(I-\lambda\gamma {P}) \right]{\epsilon} \\
& = & \left[ (I-\gamma {P})^{-1}(I-\gamma {P} + \gamma {P} - \lambda\gamma {P})-(I-\gamma P_*)^{-1}(I-\lambda\gamma {P}) \right]{\epsilon} \\
& = & \left[ \left(I+(1-\lambda)(I-\gamma{P})^{-1}\gamma{P}+\lambda(I-\gamma P_*)^{-1}\gamma{P}\right) -(I-\gamma P_*)^{-1} \right]{\epsilon} \\
& = & \left[ \left((1-\lambda)(I-\gamma{P})^{-1}\gamma{P}+\lambda(I-\gamma P_*)^{-1}\gamma{P}\right) -(I-\gamma P_*)^{-1}\gamma P_* \right]{\epsilon} \\
& = & \frac{\gamma}{1-\gamma}\left[ {B_v}-{D} \right]{\epsilon}.
\end{eqnarray*}
where
\begin{eqnarray*}
{B_v} &:=& (1-\gamma)\left((1-\lambda)(I-\gamma{P})^{-1}{P}+\lambda(I-\gamma P_*)^{-1}P \right)\\
{D} &:=& (1-\gamma)P_*(I-\gamma P_*)^{-1}.
\end{eqnarray*}
{\begin{lemma}
$B_v$ and $D$ are stochastic matrices.
\end{lemma}}
\begin{proof}
It is clear that $\|D\|=1$. Also:
\begin{eqnarray*}
\|B_v\| &= & (1-\gamma)\left(1 + \frac{(1-\lambda)\gamma}{1-\gamma}+\frac{\lambda\gamma}{1-\gamma} \right) \\
& = & (1-\gamma)\left(1 + \frac{\gamma}{1-\gamma} \right) \\
& = & 1. 
\end{eqnarray*}
Then, the first bound of Corollary~\ref{valueorpolicyconvergence} follows from the application of Lemmas~\ref{compnorm} and~\ref{fromcomptospan}.
\end{proof}


\subsection{Proof of the second inequality of Corollary~\ref{valueorpolicyconvergence} (when the policy converges)}
\label{policyconverge}

Suppose that \lpi{ }converges to some policy ${\pi}$.
Write ${P}$ the corresponding stochastic matrix and 
$$
{A^\pi}:=(1-\lambda\gamma){P}(I-\lambda\gamma {P})^{-1}.
$$
Then for some big enough $k_0$, we have:
$$
l_k \leq \sum_{j=k_0}^{k-1} \left[\frac{\gamma^{k-j}}{1-\gamma}A^\pi_{jk}{A^\pi}(I-\gamma {P})-\gamma^{k-j} (P_*)^{k-j}\right]\epsilon_j  + \O(\gamma^{k-k_0})
$$
where
$$
A^\pi_{jk} := \frac{1-\gamma}{\gamma^{k-j}}\left[ \frac{\lambda\gamma}{1-\lambda\gamma} \sum_{i=j}^{k-1}\gamma^{k-1-i}\beta^{i-j}(P_*)^{k-1-i}({A^\pi})^{i-j} + \beta^{k-j}(I-\gamma {P})^{-1}(A^\pi)^{k-1-j}\right]
$$
is a stochastic matrix (for the same reasons why $B_{jk}$ is a stochastic matrix in Lemma~\ref{fnorms}).
Noticing that
\begin{eqnarray*}
{A^\pi}(I-\gamma {P}) & = & (1-\lambda\gamma){P}(I-\lambda\gamma {P})^{-1}(I-\gamma {P}) \\
& = & (1-\lambda\gamma){P}(I-\lambda\gamma {P})^{-1}(I-\lambda\gamma{P} + \lambda\gamma{P} - \gamma {P}) \\
& = & (1-\lambda\gamma){P}(I-(1-\lambda)(I-\lambda\gamma {P})^{-1}\gamma {P}) \\
& = &  (1-\lambda\gamma){P} - \gamma(1-\lambda)A^\pi{P}
\end{eqnarray*} 
we can deduce that
\begin{eqnarray}
l_k & \leq & \sum_{j=k_0}^{k-1} \left[\frac{\gamma^{k-j}}{1-\gamma}A^\pi_{jk}\left[ (1-\lambda\gamma){P} - \gamma(1-\lambda)A^\pi{P}\right]-\gamma^{k-j} (P_*)^{k-j}\right]\epsilon_j  + \O(\gamma^{k-k_0}) \nonumber \\
& = & \sum_{j=k_0}^{k-1} \gamma^{k-j}\left[\frac{1-\lambda\gamma}{1-\gamma}A^\pi_{jk} {P} - \left[ \frac{\gamma(1-\lambda)}{1-\gamma}A^\pi_{jk}A^\pi{P}+ (P_*)^{k-j}\right]\right]\epsilon_j  + \O(\gamma^{k-k_0}) \nonumber \\
& = &\frac{1-\lambda\gamma}{1-\gamma}\sum_{j=k_0}^{k-1} \gamma^{k-j}\left[B^\pi_{jk}-B'^\pi_{jk}\right]\epsilon_j + \O(\gamma^{k-k_0}) \label{lbpc}
\end{eqnarray}
where
\begin{eqnarray*}
B^\pi_{jk}&:=& A^\pi_{jk}{P}\\
B'^\pi_{jk}&:=&\frac{1-\gamma}{1-\lambda\gamma}\left[\frac{\gamma(1-\lambda)}{1-\gamma}A^\pi_{jk}A^\pi {P}+(P_*)^{k-j}\right].
\end{eqnarray*}
{\begin{lemma}
$B^\pi_{jk}$ and $B'^\pi_{jk}$ are stochastic matrices.
\end{lemma}}
\begin{proof}
It is clear that $\|B^\pi_{jk}\|=1$. Also:
\begin{eqnarray*}
\| B'^\pi_{jk} \| &= & \frac{1-\gamma}{1-\lambda\gamma}\left(1+\frac{\gamma(1-\lambda)}{1-\gamma}\right) \\
& = & \frac{1-\gamma}{1-\lambda\gamma}\frac{1-\gamma+\gamma-\lambda\gamma}{1-\gamma} \\
& = & 1. 
\end{eqnarray*}
Then, the second bound of Corollary~\ref{valueorpolicyconvergence} follows from the application of Lemmas~\ref{compnorm} and~\ref{fromcomptospan}.

\end{proof}


\section{Proofs of Lemmas~\ref{compnorm0} and \ref{compnorm} (from componentwise bounds to span seminorm bounds)}

\label{fourlemmas}

This section contains the proofs of Lemmas~\ref{compnorm0} and~\ref{compnorm} that enable us to derive span seminorm performance bounds from the componentwise analysis developped in the previous sections. It is easy to see that Lemma~\ref{compnorm0} is a special case of Lemma~\ref{compnorm}, so we only prove the latter.

Consider the notations of Lemma~\ref{compnorm}. Write $a_{kj}:=\arg\min_a\norm{y_j-a\e}{p,\mu_{kj}}$. As $X_{jk}$ and $X'_{jk}$ are stochastic matrices, $X_k \e=X'_k \e=\e$ and we can write that: 
$$
\ls{k}|x_k| \leq K \ls{k}\sum_{j=k_0}^{k-1} \xi_{k-j}(X_{kj}-X'_{kj})(y_j-a_{kj}\e).
$$
By taking the absolute value we get
$$
\ls{k}|x_k| \leq K \ls{k}\sum_{j=k_0}^{k-1} \xi_{k-j}(X_{kj}+X'_{kj})|y_j-a_{kj}\e|.
$$
It can then be seen that
{
\begin{eqnarray*} 
&&\ls{k}\left(\norm{x_k}{p,\mu}\right)^p  \\
&=&  K^p \ls{k}\transp{\mu}\left(|x_k|\right)^p \\
& \leq & K^p \ls{k}\transp{\mu}\left[\sum_{j=k_0}^{k-1} \xi_{k-j}(X_{kj}+X'_{kj})\left(|y_j-a_{kj}\e|\right)\right]^p \\
& = & K^p \ls{k}\transp{\mu}\left[\frac{\left(\sum_{j=k_0}^{k-1} \xi_{k-j}\frac{1}{2}(X_{kj}+X'_{kj})2\left(|y_j-a_{kj}\e|\right)\right)}{\sum_{j=k_0}^{k-1}\xi_{k-j}}\right]^p\left(\sum_{j=k_0}^{k-1}\xi_{k-j}\right)^p.
\end{eqnarray*}
By using Jensen's inequality (with the convex function  $x \mapsto x^p$), we get:
\begin{eqnarray*}
&&\ls{k}\left(\norm{x_k}{p,\mu}\right)^p  \\
& \leq & K^p  \ls{k} \transp{\mu} \frac{\sum_{j=k_0}^{k-1} \xi_{k-j}\frac{1}{2}(X_{kj}+X'_{kj})\left(2|y_j-a_{kj}\e|\right)^p}{\sum_{j=k_0}^{k-1}\xi_{k-j}}\left(\sum_{j'=k_0}^{k-1}\xi_{k-j'}\right)^p \\
& = &  K^p \ls{k} \sum_{j=k_0}^{k-1} \xi_{k-j}\transp{\mu_{kj}} \left[2|y_j-a_{kj}\e|\right]^p \left(\sum_{j'=k_0}^{k-1}\xi_{k-j'}\right)^{p-1}\\
& \leq & K^p  \ls{k} \sum_{j=k_0}^{k-1} \xi_{k-j}\left[2\norm{y_j-a_{kj}\e}{p,\mu_{kj}}\right]^p {K'}^{p-1}
\end{eqnarray*}
\begin{eqnarray*}
& = & K^p {K'}^{p-1} \ls{k} \sum_{j=k_0}^{k-1} \xi_{k-j}\left[\spn{y_j}{p,\mu_{kj}}\right]^p \\
& \leq & K^p {K'}^{p-1} \ls{k} \sum_{j=k_0}^{k-1} \xi_{k-j}\left[\sup_{k' \geq j' \geq k_0}\spn{y_{j'}}{p,\mu_{{k'j'}}}\right]^p \\
& = & K^p {K'}^{p-1} K' \left[\sup_{k' \geq j' \geq k_0}\spn{y_{j'}}{p,\mu_{{k'j'}}}\right]^p \\
& = & K^p {K'}^{p}\left[\sup_{k' \geq j' \geq k_0}\spn{y_{j'}}{p,\mu_{{k'j'}}}\right]^p
\end{eqnarray*}
}
where we used $\sum_{j=k_0}^{k-1}\xi_{k-j} =K'$.

As this is true for all $k_0$, and as $k_0 \mapsto \sup_{k' \geq j' \geq k_0}\spn{y_{j'}}{p,\mu_{{k'j'}}}$ is non-increasing, the result follows.


\section{Proofs of Lemma~\ref{lemmaundiscounted} and Proposition~\ref{boundundiscounted} (analysis of the undiscounted case)}

\label{proofundiscounted}

This last section contains the Proofs of Lemma
\ref{lemmaundiscounted} and Proposition~\ref{boundundiscounted} that
provide the analysis of an undiscounted problem like Tetris.

\subsection{Proof of Lemma~\ref{lemmaundiscounted} (componentwise bound)}

First of all, the relation expressed in Equation~\ref{errorfortetris} between the loss and the stochastic matrices, which we restate here for clarity:
$$
\forall k_0, \mbox{~~}\ls{k}v_*-v^{\pi_k} \leq  \ls{k}\sum_{j=k_0}^{k-1} \delta_{k-j} \left[ {G}_{jk} - {G'}_{jk}\right] \epsilon_j,
$$
is obtained by simply rewriting the first inequality of Lemma~\ref{th} with $\gamma=1$ and $\beta=1$ (note in particular that the terms $\delta_{k-j}$ collapse through the definition of ${G}_{jk}$ and ${G'}_{jk}$).

To complete the proof of the lemma, we need to show that the matrices $G_{jk}$ and $G'_{jk}$ are
substochastic matrices.  By construction, these matrices are sum of
non-negative matrices so we only need to show that their max norm is smaller than or equal to 1.

For all $n$, write ${\cal M}_n$ the set of matrices that is defined as follows:
\begin{itemize} 
\item for all sets of $n$ policies $(\pi_1,\pi_2,\cdots,\pi_n)$,~~ $P_{\pi_1}P_{\pi_2}\cdots P_{\pi_n} \in {\cal M}_n$;
\item for all $\eta \in (0,1)$, and $(P,Q) \in {\cal M}_n \times {\cal M}_n$,~~ $\eta P + (1-\eta) Q \in {\cal M}_n$.
\end{itemize}
The motivation for introducing this set is that we have the following properties:
For all $n$, $P \in {\cal M}_n$ is a substochastic matrix such that $\norm{P}{\infty} \le \alpha^{\left\lfloor  \frac{n}{n_0}\right\rfloor}$.
We  use the somewhat abusive notation $\Pi_n$ for denoting any element of ${\cal M}_n$. For instance, for some matrix $P$, 
writing $P=a \Pi_i + b \Pi_j \Pi_k=a \Pi_i + b \Pi_{j+k}$ should be read as follows: there exists $P_1 \in {\cal M}_i$, $P_2 \in {\cal M}_j$,
$P_3 \in {\cal M}_k$ and $P_4 \in {\cal M}_{k+j}$ such that $P=a P_1 + bP_2 P_3=a P_1+bP_4$.

Recall the definition of the substochastic matrix 
$$
A_k=(1-\lambda)(I-\lambda P_k)^{-1}P_k=(1-\lambda)\sum_{i=0}^{\infty} \lambda^i \Pi_{i+1}.
$$
Let $i \le j <k$. 
It can be seen that 
\begin{align}
(P_*)^{k-1-i}A_{i+1}A_{i}...A_{j+1} &= \Pi_{k-1-i} \underbrace{\left((1-\lambda) \sum_{i=0}^{\infty} \lambda^i \Pi_{i+1}\right) \cdots \left( (1-\lambda)\sum_{i=0}^{\infty} \lambda^i \Pi_{i+1}\right)}_{i-j+1\mbox{~terms}} \nonumber \\
&= \Pi_{k-j}  \underbrace{\left( (1-\lambda)\sum_{i=0}^{\infty} \lambda^i \Pi_{i}\right) \cdots \left( (1-\lambda)\sum_{i=0}^{\infty} \lambda^i \Pi_{i}\right)}_{i-j+1\mbox{~terms}}.\label{arrgl0}
\end{align}
Now, observe that
\begin{align}
\norm{\sum_{i=0}^\infty \lambda^i \Pi_{i}}{\infty}  &\le \sum_{i=0}^\infty \lambda^i \norm{\Pi_{i}}{\infty}\nonumber \\
&\le \sum_{i=0}^\infty \lambda^i \alpha^{\left \lfloor \frac{i}{n_0}\right\rfloor}\nonumber\\
&= \sum_{j=0}^{\infty} \sum_{i=0}^{n_0-1} \lambda^{jn_0+i}\alpha^j \nonumber \\
&= \sum_{j=0}^{\infty} (\lambda^{n_0} \alpha)^j \sum_{i=0}^{n_0-1} \lambda^{i}\nonumber \\
&= \frac{1-\lambda^{n_0}}{(1-\lambda^{n_0}\alpha)(1-\lambda)}.\label{arrgl1}
\end{align}
As a consequence, writing $\eta:=\frac{1-\lambda^{n_0}}{1-\lambda^{n_0}\alpha}$, we see from Equation~\ref{arrgl0} that
$$
\norm{(P_*)^{k-1-i}A_{i+1}A_{i}...A_{j+1}}\infty \le \alpha^{\left \lfloor \frac{k-j}{n_0}\right\rfloor} \eta^{i-j+1}.
$$
Similarly, by using Equation~\ref{arrgl1} and noticing that $\frac{1-\lambda^{n_0}}{1-\lambda} \stackrel{\lambda\rightarrow 1}{\longrightarrow}n_0$, it can be seen that 
$$
\norm{(I-P_k)^{-1}A_k A_{k-1} \cdots A_{j+1}}\infty \le \frac{n_0}{1-\alpha} \alpha^{\left \lfloor \frac{k-j}{n_0}\right\rfloor} \eta^{k-j}.
$$
We are ready to bound the norm of the matrix $G_{jk}$:
\begin{align}
\norm{G_{jk}}\infty &\le \frac{ \alpha^{\left \lfloor \frac{k-j}{n_0}\right\rfloor}}{\delta_{k-j}}\left[\frac{\lambda}{1-\lambda}\sum_{i=j}^{k-1} \eta^{i-j+1} + \frac{n_0\eta^{k-j}}{1-\alpha} \right]\nonumber\\
& =  \frac{ \alpha^{\left \lfloor \frac{k-j}{n_0}\right\rfloor}}{\delta_{k-j}}\left[ \left(\frac{\lambda}{1-\lambda}\right)\eta \left(\frac{1-\eta^{k-j}}{1-\eta}\right) + \frac{n_0\eta^{k-j}}{1-\alpha} \right]\nonumber\\
& =  \frac{ \alpha^{\left \lfloor \frac{k-j}{n_0}\right\rfloor}}{\delta_{k-j}}\left[ \left(\frac{\lambda}{1-\lambda}\right)\left( \frac{1-\lambda^{n_0}}{1-\lambda^{n_0}\alpha}\right)\left(\frac{1-\eta^{k-j}}{1-\eta}\right) + \frac{n_0\eta^{k-j}}{1-\alpha} \right]\nonumber\\
& = \frac{ \alpha^{\left \lfloor \frac{k-j}{n_0}\right\rfloor}}{\delta_{k-j}}\left[\left(\frac{1-\lambda^{n_0}}{1-\lambda} \right)  \left(\frac{\lambda}{1-\lambda^{n_0}\alpha} \right)\left(\frac{1-\eta^{k-j}}{1-\eta}\right) + \frac{n_0\eta^{k-j}}{1-\alpha} \right] \nonumber \\
&=1.\nonumber
\end{align}
where we used the definition of $\eta$. Therefore $G_{jk}$ is a substochastic matrix. It trivially follows that $G'_{jk}$ is also a substochastic matrix.

\subsection{Proof of Proposition~\ref{boundundiscounted} (\lp{ }norm bound)}

In order to prove the \lp{ }norm bound of  Proposition~\ref{boundundiscounted}, we rely on the following variation of 
Lemma~\ref{compnorm}.
\begin{lemma}
If $x_k$ and $y_k$ are sequences of vectors and $X_{jk}$, $X'_{jk}$ sequences of substochastic matrices satisfying
$$
\forall k_0,\mbox{~~}\ls{k}|x_k| \leq K \ls{k}\sum_{j=k_0}^{k-1} \xi_{k-j} (X_{kj}-X'_{kj})y_j,
$$
where $(\xi_i)_{i \ge 1}$ is a sequence of non-negative weights satisfying:
$$
\sum_{i=1}^{\infty} \xi_i = K' < \infty,
$$
then, for all distribution $\mu$, 
$$
{\mu_{kj}}:=\frac{1}{2} \transp{ (X_{kj}+X'_{kj})}\mu
$$
is a non-negative vector and $\tilde\mu_{kj}:=\frac{\mu_{kj}}{\norm{\mu_{kj}}{1}}$ is a distribution, and
$$
\ls{k}\norm{x_k}{p,\mu} \leq K K' \lim_{k_0 \rightarrow \infty} \left[\sup_{k \geq j \geq k_0}\norm{y_j}{p,\tilde \mu_{kj}}\right].
$$
\end{lemma}
\begin{proof}
The proof follows the lines of that of Lemma~\ref{compnorm} in Appendix~\ref{fourlemmas}. The differences are as follows:
\begin{itemize}
\item since $X_{jk}$ and $X'_{jk}$ are substochastic matrices (and not stochastic matrices), we have in general $X_{jk}e \neq  X'_{jk}e$ and must take $a_{kj}=0$, which in turn gives an \lp{ }norm bound instead of the \lp{ }span seminorm bound;
\item to express the bound in terms of the distributions $\tilde \mu_{kj}$, we use the fact that $\mu_{kj} \le \tilde{\mu}_{kj}$ which derives from $\norm{\mu_{kj}}{1}\leq 1$ since $X_{jk}$ and $X'_{jk}$ are substochastic matrices.
\end{itemize}

~\vspace{-1.1cm}

\end{proof}
Proposition~\ref{boundundiscounted} is obtained by applying this Lemma and an anologue of Lemma~\ref{fromcomptospan} for \lp{ }norm on the componentwise bound (Lemma~\ref{lemmaundiscounted} --- see previous subsection). 
The only remaining thing that needs to be  checked is that $\sum_{i=1}^{\infty}{\delta_i}$ has the right value. This is what we do now.

Similary to Equation~\ref{arrgl1}, one can see that:
\begin{align}
\sum_{i=0}^{\infty} \alpha^{\left \lfloor \frac{i}{n_0}\right\rfloor}\eta^{i} = \frac{1-\eta^{n_0}}{(1-\eta^{n_0}\alpha)(1-\eta)}\nonumber
\end{align}
and
\begin{align}
\sum_{i=0}^{\infty} \alpha^{\left \lfloor \frac{i}{n_0}\right\rfloor}(1-\eta^{i}) &=\frac{n_0}{1-\alpha}-\frac{1-\eta^{n_0}}{(1-\eta^{n_0}\alpha)(1-\eta)}. \nonumber
\end{align}
As a consequence:
\begin{align}
\sum_{i=0}^{\infty} \delta_i &= \sum_{i=0}^{\infty}  \alpha^{\left \lfloor \frac{i}{n_0}\right\rfloor} \left(\frac{1-\lambda^{n_0}}{1-\lambda} \right)  \left(\frac{\lambda}{1-\lambda^{n_0}\alpha} \right)\left(\frac{1-\eta^{i}}{1-\eta}\right) + \frac{n_0\eta^{i}}{1-\alpha} \nonumber\\
&= \left(\frac{1-\lambda^{n_0}}{1-\lambda} \right)  \left(\frac{\lambda}{1-\lambda^{n_0}\alpha} \right)\left(\frac{\sum_{i=0}^{\infty}\alpha^{\left \lfloor \frac{i}{n_0}\right\rfloor} (1-\eta^{i})}{1-\eta}\right) + \frac{n_0\sum_{i=0}^{\infty}\alpha^{\left \lfloor \frac{i}{n_0}\right\rfloor} \eta^{i}}{1-\alpha} \nonumber \\
& =  \left(\frac{1-\lambda^{n_0}}{1-\lambda} \right)  \left(\frac{\lambda}{1-\lambda^{n_0}\alpha} \right) \left( \frac{1}{1-\eta} \right)\left( \frac{n_0}{1-\alpha}-\frac{1-\eta^{n_0}}{(1-\eta^{n_0}\alpha)(1-\eta)}\right)  \nonumber \\
& \hspace{6cm}  + \left( \frac{n_0}{1-\alpha}\right)\left(\frac{1-\eta^{n_0}}{(1-\eta^{n_0}\alpha)(1-\eta)}\right) \nonumber \\
&=  \lambda {f}(\lambda)\frac{1}{1-\eta}({f(1)} - {f}(\eta)) + {f(1)} {f}(\eta) 
\end{align}
with for all $x$, $f(x):=\frac{(1-x^{n_0})}{(1-x)(1-x^{n_0}\alpha)}$ and $f(1)=\frac{n_0}{1-\alpha}$ by continuity.
Now, we can conclude by noticing that
$$
\sum_{i=1}^{\infty} \delta_i = \sum_{i=0}^{\infty} \delta_i - \delta_0
$$
and $\delta_0=\frac{n_0}{1-\alpha}=f(1)$.

\end{appendix}


\end{document}